\documentclass[journal]{IEEEtran}
\usepackage{amsfonts}
\usepackage{cite,graphicx,amsmath,amsthm}
\usepackage{fancyhdr}
\usepackage{dsfont}
\usepackage{array,color}
\usepackage{bm}
\usepackage{float}
\usepackage{algpseudocode}
\usepackage{multirow}
\usepackage{booktabs}
\usepackage{makecell}
\usepackage{hyperref}
\allowdisplaybreaks[4]
\usepackage{amsmath,amssymb, graphicx, epstopdf,cite,enumerate,booktabs, setspace, float,stfloats,bm, multirow, lscape, caption, subcaption, graphbox, soul, color, xcolor, hyperref}
\usepackage{diagbox}
\hypersetup{hidelinks}
\usepackage{cite, amsmath,amssymb, bm}
\usepackage{bbding}
\usepackage{amsthm}
\usepackage[normalem]{ulem}
\usepackage[linesnumbered,ruled]{algorithm2e}

\newtheorem{lemma}{Lemma}

\newtheorem{proposition}{Proposition}

\newtheorem{corollary}{Corollary}

\newtheorem{property}{Property}

\newtheorem{remark}{Remark}

\newtheorem{claim}{Claim}

\begin{document}

\title{\huge
Communication Efficient Robotic Mixed Reality with Gaussian Splatting Cross-Layer Optimization
}

\author{
Chenxuan Liu$^{*}$, He Li$^{*}$, Zongze Li, Shuai Wang$^{\dag}$, Wei Xu, Kejiang Ye,\\Derrick Wing Kwan Ng,~\emph{Fellow, IEEE}, and 
Chengzhong Xu,~\emph{Fellow, IEEE}
\thanks{
A preliminary version of this paper was presented at the 2025 IEEE International Conference on Computer Communications (INFOCOM) Workshop on NetRobiCS \cite{infocom}.

This work is supported by the National Natural Science Foundation of China (Grant No. 62371444), Guangdong Basic and Applied Basic Research Foundation (No. 2023B1515130002), Guangdong Special Support Plan (No. 2021TQ06X990), and the Shenzhen Science and Technology Program (Grant No. JCYJ20241202124934046, JCYJ20220818101610023, RCYX20231211090206005, RCBS20231211090517022).

Chenxuan Liu, Shuai Wang, and Kejiang Ye are with the Shenzhen Institutes of Advanced Technology, Chinese Academy of Sciences, Shenzhen, China.
Chenxuan Liu is also with the University of Chinese Academy of Sciences. 

He Li and Chengzhong Xu are with the State Key Laboratory of IOTSC, Department of Computer and Information Science, University of Macau, Macau, China.

Zongze Li is with Peng Cheng Laboratory, Shenzhen, China.

Wei Xu is with the Manifold Tech Limited, Hong Kong, China.

Derrick Wing Kwan Ng is with the School of Electrical Engineering and Telecommunications, the University of New South Wales, Australia.

Corresponding author: Shuai Wang ({\tt\footnotesize s.wang@siat.ac.cn}). 

$^*$These authors contribute equally. 
}
}
\maketitle

\begin{abstract}
Realizing low-cost communication in robotic mixed reality (RoboMR) systems presents a challenge, due to the necessity of uploading high-resolution images through wireless channels. This paper proposes Gaussian splatting (GS) RoboMR (GSMR), which enables the simulator to opportunistically render a photo-realistic view from the robot's pose by calling ``memory'' from a GS model, thus reducing the need for excessive image uploads. 
However, the GS model may involve discrepancies compared to the actual environments. To this end, a GS cross-layer optimization (GSCLO) framework is further proposed, which jointly optimizes content switching (i.e., deciding whether to upload image or not) and power allocation (i.e., adjusting to content profiles) across different frames by minimizing a newly derived GSMR loss function. 
The GSCLO problem is addressed by an accelerated penalty optimization (APO) algorithm that reduces computational complexity by over $10$x compared to traditional branch-and-bound and search algorithms. 
Moreover, variants of GSCLO are presented to achieve robust, low-power, and multi-robot GSMR.
Extensive experiments demonstrate that the proposed GSMR paradigm and GSCLO method achieve significant improvements over existing benchmarks on both wheeled and legged robots in terms of diverse metrics in various scenarios. 
{For the first time, it is found that RoboMR can be achieved with ultra-low communication costs, and mixture of data is useful for enhancing GS performance in dynamic scenarios.} 
\end{abstract}


\section{Introduction}

Embodied artificial intelligence has greatly benefited from leveraging simulation to train and test robot learning systems \cite{li2021igibson}. 
Indeed, simulators are capable of computing robot motions, state evolutions, and sensor measurements that arise from interactions between robots and environments. 
However, a significant gap exists between the simulated and real-world data \cite{wang2022federated,morra2019building}, and it is imperative to integrate real data into simulations to enhance the realism of scenarios that can be studied.

Robotic mixed reality (RoboMR) \cite{morra2019building,delmerico2022spatial,li2024seamless}, 
which facilitates interactions between real robots and virtual agents in a shared symbiotic world, emerges as a promising solution to address the above issue.
By synchronizing digital twins across simulation servers and robotic devices, RoboMR harmonizes spatial representations, leading to a paradigm shift where previously isolated virtual and real agents can now collaborate seamlessly in a metaverse space. 
However, challenges persist as simulators at the server and robots at the edge remain physically isolated, and establishing low-cost communication between them is non-trivial \cite{morra2019building,delmerico2022spatial,li2024seamless,guo2020adaptive,bastug2017toward,chen2018vr}, due to the vast data volume of high-frequency high-resolution sensor data (e.g., images). 
Existing RoboMR resource allocation approaches \cite{zheng2016wireless,yu2004iterative,zhang2024efficient,bastug2017toward,chen2018vr,Cakir2024IDTC, Jiang2023QoE,Cheng2024resource,Zhao2024Adaptive,Pan2023joint} involve the following drawbacks: 
1) They fail in fully exploiting the memory information recorded at previous RobotMR executions; 
2) They ignore the inter-dependency between high-level data encoding/decoding at the application layer and low-level communication design at the physical layer; 
3) They are mostly tested in simulation, with a significant lack of testing on real-world RoboMR datasets.

This paper proposes the concept of Gaussian splatting (GS) RoboMR (or GSMR for short), which leverages a GS model to opportunistically generate a photo-realistic view from the robot’s pose, thus avoiding excessive uploading of images. 
This GS model acts as a ``memory system'' for RoboMR, but its rendering may involve discrepancies compared to the actual environments (i.e., termed ``memory bias''), primarily due to the variations in illuminations and changes of environments. 
To this end, we propose a GS cross-layer optimization (GSCLO) framework to jointly optimize content switching (i.e., whether to upload image or not) and transmit power (i.e., adjusting to content profiles) across different frames, by minimizing a newly derived GSMR loss. We further propose an accelerated penalty optimization (APO) algorithm to solve the GSCLO problem, which reduces computational complexity by orders of magnitude compared to conventional branch-and-bound \cite{diamond2016cvxpy} and local search \cite{neumann2007randomized} algorithms, while achieving close-to-optimal performance as evidenced by experiments.
In addition, we extend the method to the cases under imperfect channel states, power minimization criteria, and multi-robot interference.
Finally, we implement RoboMR in Robot Operation System (ROS) and evaluate the performance gain brought by GSMR and GSCLO compared to existing RoboMR systems \cite{delmerico2022spatial,li2024seamless} and resource allocation algorithms \cite{zhang2024efficient,yu2004iterative,zheng2016wireless}. 
Specifically, GSMR reduces the power consumption by $75\,\%$ compared with RoboMR, and the PSNR of GSCLO is at least $5$\,dB higher than those without cross-layer optimization.

\begin{table*}[!t]
\caption{Summary of Important Variables and Parameters}
\centering
\scalebox{1}{
\begin{tabular}{|l|l|l|}
\hline
\textbf{Symbol} & \textbf{Type} & \textbf{Description} \\
\hline
$p_{t}\in\mathbb{R}_+$  & Variable & Transmit power (in $\mathrm{Watt}$) of frame $t$. \\
$x_{t}\in\mathbb{R}_+$  & Variable & Content switching ($\in\{0,1\}$) of frame $t$. \\
\hline
$\mathbf{s}_{t}\in\mathbb{R}^3$  & Data & Robot pose $\mathbf{s}_{t}=(a_{t},b_{t},\theta_{t})$, where $(a_{t},b_{t})$ and $\theta_{t}$ are position and orientation, respectively. \\
$\mathbf{u}_{t}\in\mathbb{R}^2$  & Data & Robot action $\mathbf{u}_{t}=(v_{t},\psi_{t})$, where $v_{t}$ and $\psi_{t}$ are linear and angular velocities, respectively.\\
$(\mathbf{r}_t,\mathbf{v}_t,\mathbf{m}_t)$ & Data & Real, virtual, and MR images.
\\
$(\widehat{\mathbf{r}}_t,\widehat{\mathbf{m}}_t)$ & Data & Recovered real and MR images.
\\
$(\mathbf{c}_t,\mathbf{d}_t)$ & Data & Encoded image and mask vector.
\\
\hline
$(E_t,D_t)$ & Function & Encoder and decoder functions.
\\
$\mathcal{L}$ & Function & Loss function of the RoboMR system.
\\
\hline
$T$ & Parameter & Number of frames.
\\
$\tau$ & Parameter & Duration between consecutive time slots (in s).
\\
$(L,W)$ & Parameter & Image length and width.
\\
$(h_t,\Tilde{h}_{t},\Delta h_t)$ & Parameter & Uplink channel, estimated channel, and estimation error of frame $t$. \\
$B$ & Parameter & Bandwidth allocated to the robot (in $\mathrm{Hz}$).\\
$\sigma^2$ & Parameter & External interference plus noise power (in $\mathrm{Watt}$). \\
$R_{t}$ & Parameter & Achievable rate (in $\mathrm{bps}$) of the robot.\\
$P$ & Parameter & Power budget (in $\mathrm{Watt}$) of the robot.\\
$L_{t}$ & Parameter & GS loss of frame $t$.\\
$(I,S)$ & Parameter & Data size (in bits) of robot image and robot state, respectively.\\
\hline
\end{tabular}
}
\end{table*}

Our contributions are summarized as follows:
\begin{itemize}
    \item We introduce the concept of GS to support MR with ultra-low communication costs ($<200$\,bits/frame), leading to a paradigm shift towards mixture of data for robot learning.
    \item A cross-layer optimization framework, termed GSCLO, is proposed, which effectively mitigates the discrepancies between GS-rendered images and real-world environments, thereby improving the image quality compared to existing GS. 
    \item A robust GSCLO algorithm is proposed to ensure reliable GSMR under channel uncertainties and a closed-form GS power saving factor is derived, which theoretically quantifies the benefits brought by our approach.
    \item We implement the GSMR system in ROS on both wheeled and legged robots. 
    Extensive experimental results show significant improvements over existing schemes in terms of diverse metrics.
\end{itemize}

The remainder of this paper is organized as follows.
Section \ref{section2} reviews related work. 
Section \ref{section3} states the problem. 
Section \ref{section4} describes the GSCLO scheme and APO algorithm.
Subsequently, Section~\ref{section5} presents the robust GSCLO under channel uncertainties. 
Section~\ref{section6} presents the power minimization GSCLO and multi-robot GSCLO schemes. 
Section \ref{section7} presents experimental results. 
Finally, Section~\ref{section8} concludes this work.

\emph{Notation}.
Italic, simple bold, capital bold, and curlicue letters represent scalars, vectors, matrices, and sets, respectively.
$\nabla f$ represents the gradient of a function $f$.
$\mathcal{U}(x,y)$ represents a uniform distribution within interval $[x,y]$ and $\mathcal{CN}(0,1)$ is standard normal distribution.
$\mathbb{R}$ and $\mathbb{C}$ represent real and complex fields, respectively. Finally, $\|\cdot\|$ denotes the vector norm function and $\odot$ is the Hadamard product for element-wise matrix multiplication. 
Important variables and parameters are summarized in Table I.

\section{Related Work}\label{section2}

RoboMR systems can be categorized into non-interactive \cite{surfelgan,aads} and interactive approaches \cite{li2021igibson,delmerico2022spatial,safety-ar,dense-rl,li2024seamless}.
Non-interactive RoboMR processes off-the-shelf datasets without virtual reality synchronization, which are adopted for perception data augmentation \cite{surfelgan,aads}.
On the other hand, interactive RoboMR leads to a paradigm shift where previous isolated virtual and real agents are now synchronized for cooperation and competition in a shared scenario \cite{delmerico2022spatial,li2021igibson,safety-ar,dense-rl}.
However, interactive RoboMR systems require a large volume of sensor data sharing between the server and robot, which poses a great challenge to communication.

To address the above challenge, resource allocation for RoboMR becomes imperative \cite{bastug2017toward}. 
In contrast to conventional communication systems, RoboMR systems aim to maximize the MR fidelity instead of the communication throughput.
Therefore, RoboMR resource allocation becomes very different from traditional resource allocation schemes that merely consider the wireless channel state information \cite{yu2004iterative,zheng2016wireless}.
For instance, the water-filling scheme allocates more resources to better channels for throughput maximization \cite{yu2004iterative}, and the max-min fairness scheme allocates more resources to worse channels to maintain certain quality of service \cite{zheng2016wireless}.
Indeed, these schemes could lead to poor performance in RoboMR systems because they do not account for the MR factors such as video complexities and quality of experience (QoE).
This ignites research efforts \cite{Zhao2024Adaptive, Cakir2024IDTC, Jiang2023QoE,Cheng2024resource,Pan2023joint} on MR oriented resource allocators. 
For instance, an intelligent digital twin communication framework was proposed for DT synchronization within resource-constrained networks \cite{Cakir2024IDTC}.
Also, a joint resource allocation and service selection scheme was designed to maximize the QoE utility for MR systems \cite{Jiang2023QoE}. 
Furthermore, learning-based resource allocators have also been exploited for optimizing the frame rate, user association, and link allocation in MR systems \cite{chen2018vr,Cheng2024resource,Zhao2024Adaptive,Pan2023joint}.

Nonetheless, existing methods mainly focus on image compression and rate control paradigms, which fail in fully exploiting the memory recorded at previously RoboMR executions. 
Leveraging these memories and the recent advancements in building world models (i.e., GS \cite{kerbl20233d,keetha2024splatam}), the server is able to ``imagine'' an image from a new viewpoint adjacent to the recording position. 
Therefore, no image uploading is needed when the memory is clear enough.
Furthermore, current RoboMR resource allocation algorithms either focus on the physical layer \cite{yu2004iterative,zheng2016wireless} or the high-level layers \cite{bastug2017toward,chen2018vr,Cakir2024IDTC, Jiang2023QoE,Cheng2024resource,Zhao2024Adaptive,Pan2023joint}, without forming a unified cross-layer framework. 
Finally, current methods are mostly tested in simulation.
To verify the effectiveness and robustness of RoboMR, it is necessary to develop real-time algorithms and evaluate the methods in real-world RoboMR systems.

Here, we propose a novel memory-assisted RoboMR paradigm, termed GSMR, which leverages a GS model to avoid excessive image uploads. 
We further propose a GSCLO framework to jointly optimize physical layer and application layer variables across different frames.
Finally, we implement the GSMR paradigm and GSCLO algorithm in ROS, and aim to provide a concrete step toward RoboMR in the real world.

\begin{figure*}[t]
    \centering
    \includegraphics[width=0.98\textwidth]{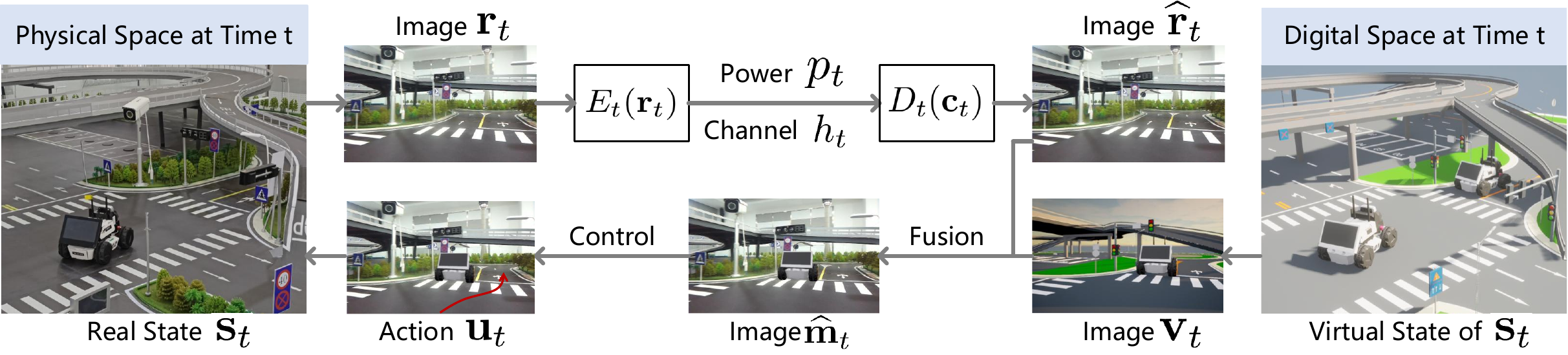}
    \caption{System model of RoboMR with cross-layer optimization of $\{E_t, D_t\}_{t=1}^T$ and $\{p_t\}_{t=1}^T$.}
    \label{fig:model}
\end{figure*}

\section{System Model and Problem Formulation}\label{section3}

\subsection{System Overview}
As shown in Fig.~1, a RoboMR system consists of a mobile robot acting as a real agent and a simulation server constructing the virtual world. 
The operation of RoboMR is discretized into $T$ time slots, where the duration of each time slot is $\tau$.
At the $t$-th time slot, $t \in\{1, \ldots, T\}$, 
the robot state is ${\mathbf{s}}_t=({a}_{t},{b}_{t},{\theta}_{t})$, where $(a_t,b_t)$ and $\theta_t$ denote the two-dimensional position and orientation, respectively. The interaction between the robot and the simulator involves four consecutive stages: 
\begin{itemize}
    \item[1)] \emph{Sensing}: A pair of real $\mathbf{r}_t\in\mathbb{R}^{3LW}$ and virtual $\mathbf{v}_t\in\mathbb{R}^{3LW}$ camera images are generated at the robot and the server, respectively, where $L$ and $W$ are the image length and width respectively, and the coefficient $3$ accounts for the red green blue (RGB) channels.
    \item[2)] \emph{Communication}: The robot uploads $\mathbf{r}_t$ to the server, which is detailed in Section III-B.   
    \item[3)] \emph{Fusion}: The image $\mathbf{r}_t$ is aggregated with $\mathbf{v}_t$ to generate MR image $\mathbf{m}_t= 
\mathbf{d}_t\odot\mathbf{r}_t+
    (\mathbf{1}-\mathbf{d}_t)
\odot\mathbf{v}_t$, where $\mathbf{d}_t$ is the mask (binary vector) of a real image for extracting the background and $(\mathbf{1}-\mathbf{d}_t)$ is the mask of a virtual image for extracting the virtual object.
    \item[4)] \emph{Control}: The MR image $\mathbf{m}_t$ is fed to the robot for generating control command 
$\mathbf{u}_t=(v_t,\phi_t)$, where $v_t$ and $\phi_t$ represent the linear and angular velocities, respectively.\footnote{Note that robot control $\mathbf{u}_t$ needs to avoid any potential collision with both virtual and real agents in the MR image $\mathbf{m}_t$.} 
\end{itemize}

After the above interaction, the subsequent state is 
${\mathbf{s}_{t + 1}} = {\mathbf{s}_t} + F({\mathbf{s}_t},{\mathbf{u}_t})\tau$, where $F({\mathbf{s}_t},{\mathbf{u}_t})$ is the discrete-time kinematic model of the robot \cite{han2023rda}.
This completes one loop of RoboMR and the time index is updated as $t\leftarrow t+1$. 
The entire procedure terminates when $t=T$, and the robot learning dataset collected in RoboMR is $\{\mathbf{m}_t,\mathbf{u}_t\}_{t=1}^T$.

\subsection{Communication Model}

To generate $\mathbf{m}_t$, streaming data $\{\mathbf{r}_t\}_{t=1}^T$ needs to be uploaded over wireless communication, as the robot and simulator are physically isolated.
This requires a pair of image encoder and decoder functions $(E_t,D_t)$, where the encoded vector $\mathbf{c}_t=E_t(\mathbf{r}_t)$ and the decoded vector
$\widehat{\mathbf{r}}_t=D_t(\mathbf{c}_t)$, 
such that $\widehat{\mathbf{r}}_t$ is close to $\mathbf{r}_t$.
Accordingly, the MR image is 
\begin{align}
\widehat{\mathbf{m}}_t(E_t,D_t)= 
\mathbf{d}_t\odot \widehat{\mathbf{r}}_t+
    (\mathbf{1}-\mathbf{d}_t)
\odot\mathbf{v}_t. \label{mthat}
\end{align}
This $\widehat{\mathbf{m}}_t$ may differ from the desired image $\mathbf{m}_t$. 
To measure their difference, we adopt loss function 
 \begin{align}\label{LGS}
\mathcal{L}\left(\mathbf{m}_t,\widehat{\mathbf{m}}_t\right)
=(1-\lambda)\|\mathbf{m}_t-\widehat{\mathbf{m}}_t\|_1 +
  \lambda \, {\mathsf{DS}}(\mathbf{m}_t,\widehat{\mathbf{m}}_t),
\end{align}
where the weight $\lambda\in[0,1]$ is set to $\lambda=0.2$ according to \cite{kerbl20233d} and the DSSIM function ${\mathsf{DS}}$ is 
\begin{align}
{\mathsf{DS}}(\mathbf{m}_t,\widehat{\mathbf{m}}_t)
=1-\mathsf{SSIM} (\mathbf{m}_t,\widehat{\mathbf{m}}_t),
\end{align}
with $\mathsf{SSIM}$ being the structural similarity index measure (SSIM) function detailed in \cite[Eqn. 5]{wang2011ssim}.

The wireless channels are assumed to be quasi-static during each time slot, and vary in different time slots. 
Let $h_{t}\in\mathbb{C}$ and $|h_{t}|^2$ denote the uplink channel coefficient and gain from the robot to the server at time slot $t$. 
The channel gains $\{|h_{t}|^2\}_{t=1}^T$ can be pre-determined as detailed in Section V.
The achievable rate in bits/s between the robot and the server at time slot $t$ is 
\begin{align}
R_{t}(p_{t})=B\mathrm{log}_2\left(1+\frac{|h_{t}|^2p_{t}}{\sigma^2}\right),
\end{align}
where $p_{t}$ and $B$ denote the transmit power and bandwidth of the robot at time slot $t$, respectively, and $\sigma^2$ denotes the power of external interference plus additive white Gaussian noise (AWGN) at the server. 
Let $C_t$ denote the data volume of $\mathbf{c}_t$, which is a function of the encoder $E_t$. To guarantee successful transmission of $\mathbf{c}_t$, we must have $\tau R_{t}(p_{t})\geq C_t(E_t), \forall t$.

\subsection{Cross-Layer Optimization}

Having the communication workload satisfied, our goal is to minimize the average distortion between the ground truth MR images $\{\mathbf{m}_t\}_{t=1}^T$
and the actual MR images $\{\widehat{\mathbf{m}}_t\}_{t=1}^T$, by jointly planning the power $\mathcal{P}=\{p_{t}\}_{t=1}^T$ ($\mathbf{p}=[p_{1},\cdots,p_{T}]^T$) across all time slots and designing encoder-decoder functions $\mathcal{F}=\{(E_t,D_t)\}_{t=1}^T$. 
This is realized by solving the following cross-layer optimization problem:
	\begin{subequations}
		\label{P1}
		\begin{align}
			\mathsf{P}:\,\,\,
			\min_{\mathcal{P},\mathcal{F}
			}~&\frac{1}{T}\sum_{t=1}^T\mathcal{L}\left(\mathbf{m}_t,\widehat{\mathbf{m}}_t(E_t,D_t)\right) \label{Pa}  \\
			\textrm{s.t.} ~~ & 
\tau B\mathrm{log}_2\left(1+\frac{|h_{t}|^2p_{t}}{\sigma^2}\right)\geq C_t(E_t), \ \forall t,     \label{Pb}   \\
   &\frac{1}{T} \sum_{t=1}^{T}p_{t} \leq P, 
   \ 
   p_{t}\geq 0, \ \forall t,
   \label{Pc}
		\end{align}
	\end{subequations}
where constraint \eqref{Pb} is the transmission requirement, and \eqref{Pc} is the power constraint, with $P$ being the power budget.

Problem $\mathsf{P}$ is defined as the \emph{implicit} CLO problem, since it involves the implicit functions $\{E_t,D_t\}$. Existing methods design $\{E_t,D_t\}$ exploiting compression (e.g., JPEG) or caching \cite{guo2020adaptive}. 
However, they still lead to high data rate requirements, e.g., each $C_t$ is in the range of hundreds of KBytes or more.

\section{Gaussian Splatting RoboMR}\label{section4}

This paper proposes GSMR, which adopts a GS model as a ``memory system'' for RoboMR.
This enables the server to opportunistically generate a photo-realistic image from a pose \emph{key} as shown in Fig.~\ref{fig:GSMR}. 
Since this key is only tens of Bytes, the data rate requirements $\{C_t\}$ could be reduced by orders of magnitude compared to existing approaches.

\begin{figure}[t]
    \centering
    \includegraphics[width=0.48\textwidth]{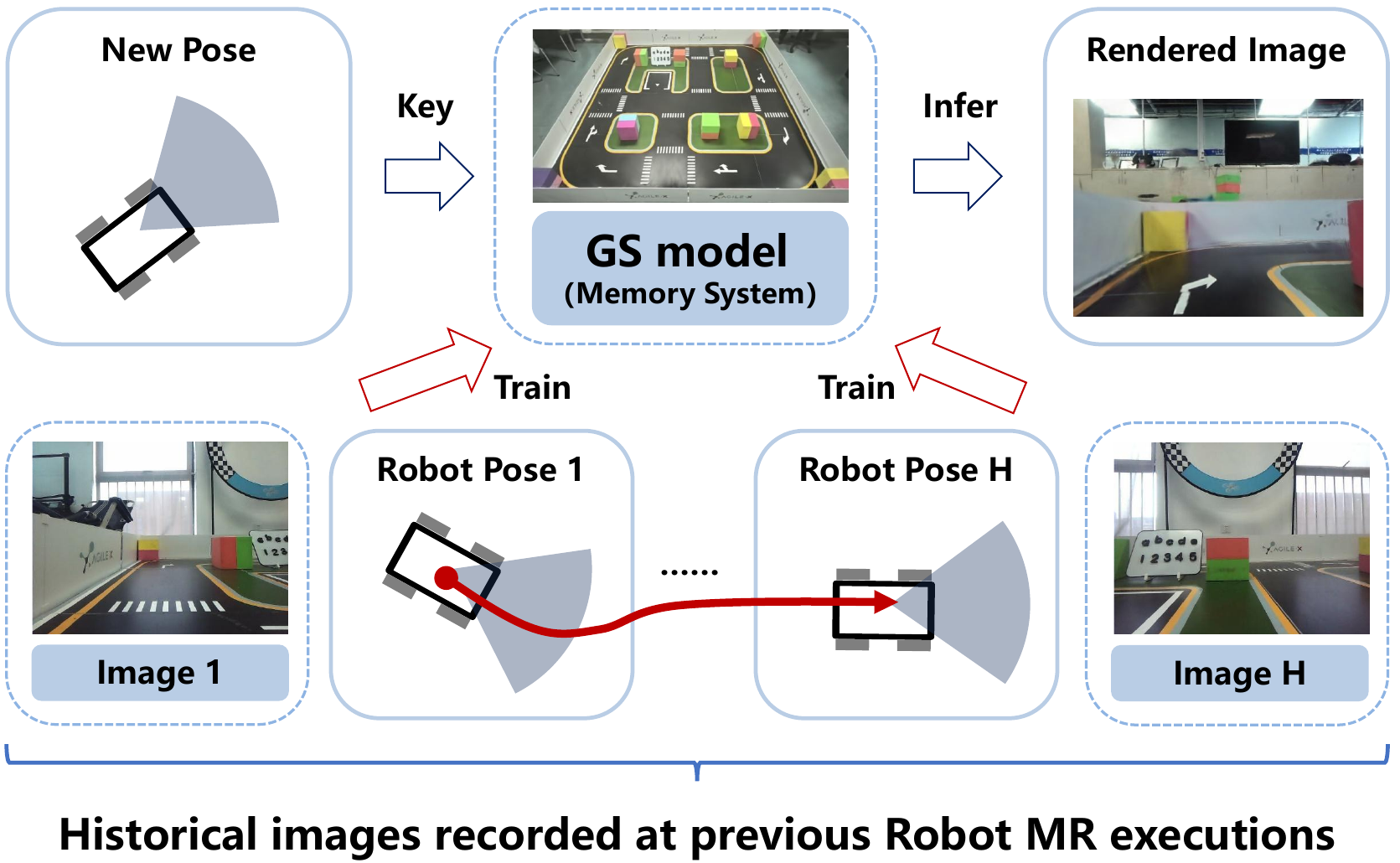}
    \caption{Leverage GS model to render images.}
    \label{fig:GSMR}
\end{figure}

\subsection{Gaussian Splatting and Content Switching}

In particular, the GS model, $\Phi$, is operated on a given robot pose $\mathbf{s}_t$, and outputs a synthesis image $\mathbf{y}_t=\Phi(\mathbf{s}_t)$. 
The model $\Phi$ is obtained as follows.
\begin{itemize}
\item First, we represent the scene as numerous 3D Gaussians 
$\mathcal{S}=\{\mathcal{G}_1,\mathcal{G}_2,\cdots\}$, where each $\mathcal{G}_i$ is characterized by a set of trainable parameters.
\item Next, with Gaussians $\mathcal{S}$, RGB images can then be efficiently rendered by alpha-compositing the splatted 2D projection of each Gaussian in order in pixel space. 
The GS-rendered image is represented by a vector $\mathbf{y}_t$. 
\item Finally, given a set of GS generated pose-image pairs $\{(\mathbf{s}_t,\mathbf{y}_t)\}_{t=1}^{H}$ and ground truth pose-image pairs $\{(\mathbf{s}_t, \mathbf{y}^*_t)\}_{t=1}^{H}$, where $H$ is the number of training data, we train GS parameters $\mathcal{S}$ through back propagation of the gradient of loss function $\mathcal{L}(\mathbf{y}_t,\mathbf{y}^*_t)$ with respect to $\mathcal{S}$.
\end{itemize}

In general, the rendered image $\mathbf{y}_t=\Phi(\mathbf{s}_t)$ may mismatch the ground truth image $\mathbf{y}^*_t$.
To address such mismatch, 
we further propose a content switching scheme shown in Fig.~\ref{fig:GS}.
Consequently, the proposed method designs the encoder function $E_t$ as 
	\begin{align}\label{E}
		\mathbf{c}_t=E_t(\mathbf{r}_t) = x_t\mathbf{r}_t + (1-x_t)\mathbf{s}_t, 
	\end{align}
where $x_t\in\{0,1\}$ denotes the content switching variable to be optimized, $x_t=1$ represents transmitting image, and $x_t=0$ represents transmitting pose. 
Accordingly, the decoder function is
\begin{align}\label{D}
		\widehat{\mathbf{r}}_t=D_t(\mathbf{c}_t) = x_t\mathbf{r}_t + (1-x_t)\Phi(\mathbf{s}_t).
\end{align}
The loss function 
$\mathcal{L}\left(\mathbf{m}_t,\widehat{\mathbf{m}}_t\right) $ in $\mathsf{P}$ is thus 
\begin{align}
&\mathcal{L}\left(\mathbf{m}_t,\widehat{\mathbf{m}}_t\right)
=
L_t
(1-x_t)
:=\Theta_t(x_t),
\end{align}
where we have defined 
\begin{align}
L_t:=\mathcal{L}\Big(
&
\mathbf{d}_t\odot\mathbf{r}_t+
(\mathbf{1}-\mathbf{d}_t)
\odot\mathbf{v}_t, 
\nonumber\\
&
\mathbf{d}_t\odot \Phi(\mathbf{s}_t) +
(\mathbf{1}-\mathbf{d}_t)
\odot\mathbf{v}_t
\Big).
\end{align}
On the other hand, since the data volume is $C_t(E_t)=x_t I + (1-x_t)S$, where $I$ and $S$ (with $S\ll I$) are the data volumes in bits of image and pose, respectively, constraint \eqref{Pb} becomes
\begin{align}\label{Pb2}
\tau B\mathrm{log}_2\left(1+\frac{|h_{t}|^2p_{t}}{\sigma^2}\right)\geq x_t I + (1-x_t)S,\ \forall t.
	\end{align}

\subsection{GSCLO Problem Formulation} 

Based on the scheme in Section IV-A, problem $\mathsf{P}$ is explicitly formulated as 
	\begin{subequations}
		\label{P2}
		\begin{align}
			&\mathsf{P}_{\mathrm{GS}}:\,\,\,
			\min_{\mathcal{P},\mathcal{X}
			}~\frac{1}{T}\sum_{t=1}^T\Theta_t(x_t) \label{P2a}  \\
			\textrm{s.t.} ~~ & \tau B\mathrm{log}_2\left(1+\frac{|h_{t}|^2p_{t}}{\sigma^2}\right)\geq x_t I + (1-x_t)S,\ \forall t, \label{Pc_GS}
 \\  
     &\frac{1}{T} \sum_{t=1}^{T}p_{t} \leq P, 
   \ 
   p_{t}\geq 0, \ \forall t,
   \label{Pd_GS}
   \\
   & x_t\in\{0,1\}, \ \forall t, \label{Px_GS}
		\end{align}
	\end{subequations}
where $\mathcal{X}=\{x_{t}\}$ ($\mathbf{x}=[x_1,\cdots,x_T]^T$).

Problem $\mathsf{P}_{\mathrm{GS}}$ is defined as the GSCLO problem, since it leverages the GS model to design functions $\{E_t,D_t\}$ in the CLO problem $\mathsf{P}$.
Note that $\mathsf{P}_{\mathrm{GS}}$ is a mixed integer nonlinear programming (MINLP) problem, which can be solved by branch-and-bound (B\&B) via Mosek \cite{diamond2016cvxpy}. 
However, the complexity is exponential in $T$, which is time-consuming when the number of frames $T$ is large. 
On the other hand, $\mathsf{P}_{\mathrm{GS}}$ can be addressed by continuous relaxation and rounding \cite{hubner2014rounding}, but this would lead to non-negligible performance loss.
Lastly, existing power allocation algorithms, e.g., water-filling \cite{yu2004iterative}, max-min fairness \cite{zheng2016wireless}, fail in directly minimizing the GSMR loss in $\mathsf{P}_{\mathrm{GS}}$.

\begin{figure}[t]
    \centering
    \includegraphics[width=0.49\textwidth]{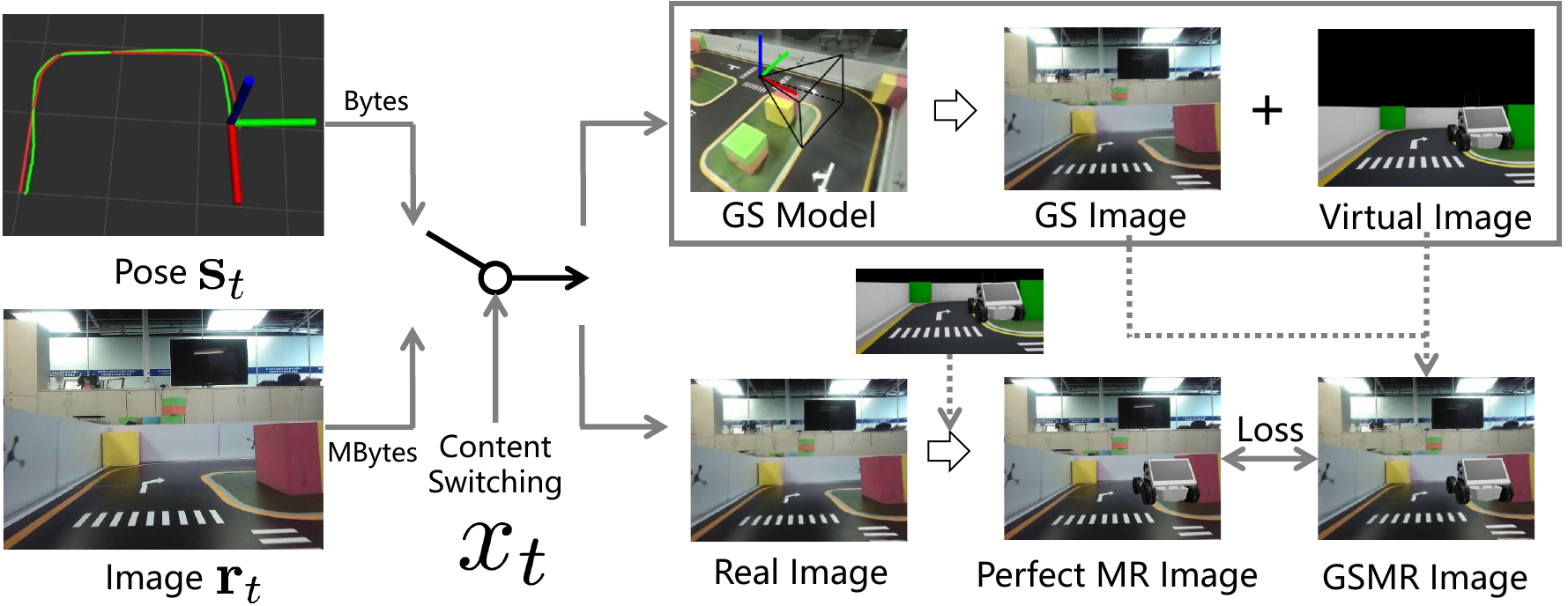}
    \caption{Architecture of content switching GSMR.}
    \label{fig:GS}
\end{figure}
 
\subsection{Penalty Continuous Reformulation}
To tackle the discontinuity, we first relax the binary constraint $x_{t}\in\{0,1\}$ into a linear constraint $0\leq x_{t}\leq 1$, $\forall t$. 
However, the relaxation is not tight and the solution to the relaxed problem could be $0<x_{t}<1$. 
To promote a binary solution for the relaxed variable $\{x_{t}\}$, we augment the objective function with a penalty term as in \cite{rinaldi2009new}. Accordingly, the approximate reformulation with regularized penalty term of $\mathsf{P_{GS}}$ is given by
	\begin{subequations}
		\label{P3}
		\begin{align}
			\mathsf{P}_{\mathrm{Penalty}}:\,\,\,
			\min_{\mathcal{P},\mathcal{X}
			}~&\frac{1}{T}\sum_{t=1}^T\Theta_t(x_t) + \varphi(\mathbf x) \label{PPenaltya}  \\
			\textrm{s.t.} ~~ & 
            \textsf{constraints }(\ref{Pc_GS}), (\ref{Pd_GS}), 
               \\
   & 0\leq x_{t}\leq 1, \ \forall t, \label{0<x<1}
		\end{align}
	\end{subequations}
where $\varphi(\mathbf x)$ is a penalty function to penalize the violation of the zero-one integer constraints. A celebrated penalty function was introduced in \cite{lucidi2010exact}, where the penalty function is set as
\begin{align} \label{eq:penaltyterm}
    \varphi(\mathbf x)=\frac{1}{\beta}\sum_{t=1}^T x_{t}(1-x_t),
\end{align}
where $\beta>0$ is the penalty parameter.
According to \cite[Proposition 1]{lucidi2010exact}, with the penalty term in (\ref{eq:penaltyterm}), there exists an upper bound $\bar\beta>0$ such that for any $\beta\in[0,\bar\beta]$, $\mathsf{P}_{\mathrm{Penalty}}$ and $\mathsf{P_{GS}}$ have the same minimum points, i.e., $\mathsf{P}_{\mathrm{Penalty}}$ and $\mathsf{P_{GS}}$ are equivalent with a proper choice of $\beta$ \cite{lucidi2010exact}.

\subsection{Difference of Convex Algorithm}

It can be seen from problem $\mathsf{P}_{\mathrm{Penalty}}$ that the constraints in \eqref{Pc_GS} are convex, since the function $-\mathrm{log}_2(\cdot)$ is convex.
Besides, the constraints in \eqref{Pd_GS} are all affine. 
The only nonconvex term in the objective function is the penalty $\varphi(\mathbf x)$. 
However, by expanding $\varphi(\mathbf x)$ as 
$\varphi(\mathbf x) =
\frac{1}{\beta}\sum_{t=1}^T x_{t}- \frac{1}{\beta}\sum_{t=1}^T x_{t}^2
$, we find that $\varphi(\mathbf x)$ is a difference of convex (DC) function \cite{an2005dc}, and such DC functions enjoy excellent surrogate properties. 
Specifically, given a certain solution $\{x_{t}^{\star}\}$,
we apply the first-order Taylor expansion on $\varphi(\mathbf x)$ and obtain 
$\varphi(\mathbf x)\approx \widehat{\varphi}(\mathbf x|\mathbf x^\star )$, where 
\begin{align}
&\widehat{\varphi}(\mathbf x|\mathbf x^\star ) = 
\sum_{t=1}^T 
\left(
\frac{1}{\beta}x_{t}-\frac{2}{\beta}x_{t}^{\star}x_{t}+\frac{1}{\beta}x_{t}^{\star^2}
\right).
\end{align}
Using the properties of DC functions, the following proposition can be established.

\begin{proposition}
$\widehat{\varphi}(\mathbf x|\mathbf x^\star )$ satisfy the following conditions:

\noindent(i) Convexity: $\widehat{\varphi}(\mathbf x|\mathbf x^\star )$ is convex in $\mathbf{x}$.

\noindent(ii) Upper bound: $\widehat{\varphi}(\mathbf x|\mathbf x^\star )
\geq \varphi(\mathbf x)$ for any $\mathbf x$.

\noindent(iii) 
Local equivalence:
$\widehat{\varphi}(\mathbf x^\star|\mathbf x^\star )
=\varphi(\mathbf x^\star)$ 
and 
$\nabla\widehat{\varphi}(\mathbf x^\star|\mathbf x^\star )
=\nabla\varphi(\mathbf x^\star)$.

\end{proposition}
\begin{proof}
Part (i) is proved by checking the semi-definiteness of the Hessian of $\widehat{\varphi}$.
In particular, $\nabla^2_{\mathbf x}\widehat{\varphi}=\mathbf{0}$, which is semi-definite.
Part (ii) is proved by checking 
\begin{align}
\widehat{\varphi}(\mathbf x|\mathbf x^\star )-\varphi(\mathbf x)
&=
\frac{1}{\beta}\sum_{t=1}^T x_{t}^2
-
\sum_{t=1}^T\frac{2}{\beta}x_{t}^{\star}x_{t}
+\sum_{t=1}^T
\frac{1}{\beta}x_{t}^{\star^2}
\nonumber
\\
&=
\frac{1}{\beta}\sum_{t=1}^T
\left(x_{t}-x_{t}^{\star}\right)^2\geq 0.
\end{align}
Part (iii) is proved by computing and comparing the function and gradient values of  $\widehat{\varphi}$ and $\varphi$.
\end{proof}

Based on part (ii) of \textbf{Proposition 1}, an upper bound problem, denoted as $\mathsf{P}_{\mathrm{Penalty}}'$, can be directly
obtained if we replace the function $\varphi(\mathbf x)$ by $\widehat{\varphi}(\mathbf x|\mathbf x^\star )$ expanded around a feasible point $\mathbf x^\star$ for the problem $\mathsf{P}_{\mathrm{Penalty}}$.
Moreover, according to part (i) of \textbf{Proposition 1}, this $\mathsf{P}_{\mathrm{Penalty}}'$ is guaranteed to be a solvable convex problem.
Finally, according to part (iii) of \textbf{Proposition 1}, a tighter upper bound can be achieved if we treat the obtained solution as another feasible point and continue
to construct the next round surrogate function. 
This leads to the DC algorithm \cite{an2005dc}, which solves a sequence of convex optimization problems 
$\{\mathsf{P}_{\mathrm{DC}}^{[0]},\mathsf{P}_{\mathrm{DC}}^{[1]},\mathsf{P}_{\mathrm{DC}}^{[2]},\cdots\}$, where $\mathsf{P}_{\mathrm{DC}}^{[n+1]}$ is the optimization problem in the $(n+1)$-th iteration of the DC algorithm, and is given by 
		\begin{align}\label{Pt+1}
			\mathsf{P}_{\mathrm{DC}}^{[n+1]}:\,\,\,
			\min_{\mathcal{P},\mathcal{X}
			}~&\frac{1}{T}\sum_{t=1}^T\Theta_t(x_t)
   +
   \widehat{\varphi}(\mathbf x|\mathbf x^{[n]} ) \nonumber
   \\
			\textrm{s.t.} ~~ & \textsf{constraints }(\ref{Pc_GS}), (\ref{Pd_GS}), (\ref{0<x<1})
		\end{align}
Here, $\mathcal{X}^{[n]}=\{x_{t}^{[n]}\}$ is the optimal solution of $\mathcal{X}$ to $\mathsf{P}_{\mathrm{DC}}^{[n]}$.
Each $\mathsf{P}_{\mathrm{DC}}^{[n+1]}$ is a convex problem and can be solved via off-the-shelf toolbox (e.g., CVXPY) with a complexity of $\mathcal{O}((2T)^{3.5})$.
According to \textbf{Proposition 1} and \cite{abbaszadehpeivasti2024rate}, any limit point of the sequence $\{(\mathcal{P}^{[n]}, \mathcal{X}^{[n]}\}_{n=0,1,\cdots}$ is a Karush-Kuhn-Tucker (KKT) solution to the problem $\mathsf{P}_{\mathrm{Penalty}}$ for any feasible starting point 
$(\mathcal{P}^{[0]}, \mathcal{X}^{[0]})$.

\subsection{Initialization via Ranking}

The APO algorithm requires a feasible initialization. 
Initialization using $\{x_{t}=\delta,\forall t\}$ with $\delta>0$ (e.g., $\delta=0.5$) may lead to violations of constraints \eqref{Pc_GS} and \eqref{Pd_GS}.
A conservative initialization is $\{x_{t}=0,\forall t\}$. 
However, this easily leads to slow convergence of the DC algorithm. 
Here, we propose a ranking-based algorithm to accelerate the convergence while guaranteeing feasibility.

Specifically, the ranking algorithm sorts the list of GS losses $\mathcal{R}=\{L_t\}_{t=1}^T$ in descending order, and the reordered list is denoted as $\mathcal{R}'=\{L_{g(1)},L_{g(2)},\cdots\}$, where $g(j)$ is an index mapping $j\rightarrow g(j)$ such that 
$L_{g(j)}\geq 
L_{g(j+1)}$ for all $j$ and $g(1)$ is the index corresponding to the largest value among all $\{L_t\}$. 
The ranking solution is given by:
\begin{equation}\label{ranking}
(x_t^{\mathrm{Rank}},p_t^{\mathrm{Rank}})=
\left\{\begin{array}{ll}
(1,\frac{\sigma^2}{|h_{t}|^2}
2^{
\frac{I}{\tau B}
}-\frac{\sigma^2}{|h_{t}|^2}), \ &\mathrm{if}~g(t)\leq \mu \\
(0,\frac{\sigma^2}{|h_{t}|^2}
2^{
\frac{S}{\tau B}
}-\frac{\sigma^2}{|h_{t}|^2}), \ &\mathrm{if}~g(t)> \mu
\end{array}\right.,
\end{equation} 
where 
$\mu=
\mathrm{arg} \mathop{\mathrm{max}}_{x\in\mathbb{Z}}
\{x:
\Xi(x)
\leq
TP\}
$ and 
\begin{align}
\Xi(x)=
\sum_{t=1}^x\frac{\sigma^22^{
\frac{I}{\tau B}
}-\sigma^2}{|h_{g(t)}|^2}
+
\sum_{t=x+1}^T
\frac{\sigma^22^{
\frac{S}{\tau B}
}-\sigma^2}{|h_{g(t)}|^2}.
\end{align}
The insight of solution \eqref{ranking} is that $L_t$ represents how much loss can be reduced by switching from pose to image, and $x_t=1$ should be assigned to the frame that may reduce the loss to the maximum extent. 
Moreover, given the values of $x_t=1$ or $x_t=0$, the optimal transmit power at time $t$ must activate constraint \eqref{Pd_GS}, so as to reserve more power for the remaining frames.
The following proposition can be established to confirm the above insight.

\begin{proposition}
The ranking-based solution $\{x_t^{\mathrm{Rank}},p_t^{\mathrm{Rank}}\}$ is optimal to $\mathsf{P_{GS}}$ if $|h_1|^2=\cdots=|h_T|^2$.
\end{proposition}
\begin{proof}
First, it can be proved that constraint \eqref{Pd_GS} always activates at the optimal $\{x_t^*,p_t^*\}$.
This gives $p_t^* = \frac{\sigma^2}{|h_{t}|^2}
2^{
\frac{I}{\tau B}
}-\frac{\sigma^2}{|h_{t}|^2}$ if $x_t^*=1$ and 
$p_t^* = \frac{\sigma^2}{|h_{t}|^2}
2^{
\frac{S}{\tau B}
}-\frac{\sigma^2}{|h_{t}|^2}$ if $x_t^*=0$.
Putting the above equations and $|h_1|^2=\cdots=|h_T|^2$ into $\mathsf{P_{GS}}$, we can prove $\sum x_t^*=\mu$ by contradiction. Assuming $\sum x_t^*<\mu$, then we can always decrease the objective value $\sum_tL_t(1-x_t)$ by setting $\sum x_t^*\leftarrow\sum x_t^*+1$ and $x_t^* \leftarrow 1-x_t^*$ for any $x_t^*=0$.
This contradicts the optimality of $\{x_t^*,p_t^*\}$.
\end{proof}

To efficiently compute the threshold $\mu$, it is observed that $\Xi(x)$ is an increasing function of $x$. Therefore, $\mu$ can be obtained from checking $\Xi(x)
\leq
TP$ using bisection method within
interval $[1,T]$. The bisection method exhibits an iteration complexity of
$\mathcal{O}(\mathrm{log}(T))$.

\textbf{Proposition 2} shows that the performance of ranking-based algorithm depends on the variation of $\{h_t\}$. It is in general suboptimal to $\mathsf{P}_{\mathrm{GS}}$, since it assumes equal channel gains.
Based on \textbf{Proposition 2}, the proposed initialization for APO is $\{x_{t}^{[0]}=x_{t}^{\mathrm{Rank}},p_t^{[0]}=p_t^{\mathrm{Rank}}\}$, and the entire APO algorithm is summarized in Algorithm 1.

\begin{algorithm}[!t]
    \caption{GSCLO via APO.}
        \begin{algorithmic}[1]
            \State \textbf{Input} Parameters $(\{L_t\},\tau,B,\{h_t\},\sigma^2,I,S,P)$.
            \State \textbf{Initialize} $\{x_t^{[0]},p_t^{[0]}\}$ according to \eqref{ranking} and set $n=0$.
            \State \textbf{Repeat}
            \State \ \ \   Solve $\mathsf{P}_{\mathrm{DC}}^{[n+1]}$ and obtain its solution $\{x_t^{*},p_t^{*}\}$.
            \State \ \ \   $\{x_t^{[n+1]},p_t^{[n+1]}\}\leftarrow \{x_t^*,p_t^*\}$ and $n \leftarrow  n+1$.
            \State \textbf{Until} convergence and set $\{\mathbf{x}^\diamond=\mathbf{x}^{[n]},\mathbf{p}^\diamond=\mathbf{p}^{[n]}\}$.
            \State \textbf{Output} $(\mathbf{x}^\diamond,\mathbf{p}^\diamond)$.
        \end{algorithmic}
\end{algorithm}

\section{Robust GSCLO Under Channel Uncertainty}\label{section5}

To realize GSCLO in practical RoboMR systems, we need to estimate the channel gains $|h_{t}|^2$. 
This implies that the parameters $|h_{t}|^2$ in $\mathsf{P}_{\mathrm{GS}}$ may involve estimation errors. 
However, the APO algorithm in Section IV is based on the assumption of perfect channel estimation, which may lead to performance degradation. 
Hence, this section investigates the robust GSCLO problem in RoboMR networks by considering the uncertainty of channel gain. 

Specifically, under channel uncertainty, the actual channel is given by
\begin{equation}
h_{t} = \Tilde{h}_{t} + \Delta h_t, \label{uncertainty}
\end{equation}
where $\Tilde{h}_{t}$ is the estimated channel known at the server, and $\Delta h_t$ is the estimation error, which is independent of $h_{t}$. 
According to \cite{Dig_Fading05}, $\Delta h_t$ can be modeled as Gaussian distribution $\Delta h_t\sim \mathcal{CN}(0,\omega^2)$ with zero mean and variance $\omega^2$.
By measuring the robot-server channel samples, the server can acquire the value of $\omega^2$.
Therefore, given channel estimate $\Tilde{h}_{t} $, $h_t$ can be regarded as a complex Gaussian variable with mean $\Tilde{h}_{t}$ and variance $\omega^2$. Therefore, $|h_t|^2$ given $\Tilde{h}_{t}$ follows a non-central $\chi^2$ distribution with two degrees of freedom, and the conditional PDF is 
\begin{align}\label{eq:pdf-fCSI}
    f(|h_t|^2|\Tilde{h}_{t})=&\frac{1}{\omega^2}\exp\left(-\frac{|\Tilde{h}_{t}|^2+|h_t|^2}{\omega^2}\right)I_0\left(\frac{\sqrt{|h_t|^2|\Tilde{h}_{t}|^2}}{\omega^2/2}\right),
    \nonumber\\
    &\forall |h_t|^2\geq 0,   
    \end{align}
where $I_0(x)$ is the zero-th order modified Bessel function of the first kind.

The channel gain distribution \eqref{eq:pdf-fCSI} implies that the achievable rate $R_t(p_t)$, which is conditioned on $|h_t|^2$, is also random~\cite{J_ZZ20SecureProb,TWC_single22}.
Consequently, the data transmission outage event occurs at the server when constraint~\eqref{Pc_GS} is not satisfied. Mathematically, this outage event can be characterized by the outage probability (OP), which is given by
\begin{align}\label{eq:OP_inter}
   P_{t}^\mathrm{out}
&=\mathrm{Pr}\left\{|h_{t}|^2<\left[2^{\frac{x_t I + (1-x_t)S}{\tau B}}-1\right]\frac{\sigma^2}{p_{t}}\Big|\Tilde{h}_{t}\right\}\\
&=\int_0^{\left[2^{\frac{x_t I + (1-x_t)S}{\tau B}}-1\right]\frac{\sigma^2}{p_{t}}}  f(|h_t|^2|\Tilde{h}_{t})d |h_t|^2\\
&=
\underbrace{
1-Q_1\left(\frac{\sqrt{|\Tilde{h}_t|^2}}{\sqrt{\omega^2/2}},\frac{\sqrt{(2^{\frac{x_t I + (1-x_t)S}{\tau B}}-1)\frac{\sigma^2}{p_{t}}}}{\sqrt{\omega^2/2}}\right)}
_{:=\Gamma_t(x_t,p_t)},
\end{align} 
where $Q_1(x,y)$ is the first-order Marcum Q-function~\cite[eqn. 4-33]{Dig_Fading05}.
With the introduction of OP, the nominal transmission constraint~\eqref{Pc_GS} becomes forcing the OP below a certain threshold, i.e., $P_{t}^\mathrm{out}\leq \epsilon$, where $\epsilon\in(0,1)$ is the target OP requirement \cite{TWC_single22}. 
This new OP-based data transmission constraint is written as $\Gamma_t(x_t,p_t)\leq \epsilon$.

Hence, under channel uncertainty, the robust GSCLO problem is 
\begin{subequations}
		\label{P3}
		\begin{align}
			\mathsf{P}_{\mathrm{GS}}':\,\,\,
			\min_{\mathcal{P},\mathcal{X}
			}~&\frac{1}{T}\sum_{t=1}^T\Theta_t(x_t)  \\
			\textrm{s.t.} ~~ 
&\Gamma_t(x_t,p_t)\leq \epsilon, \ \forall t, \label{Pc_GS'}
 \\  
     &\frac{1}{T} \sum_{t=1}^{T}p_{t} \leq P, 
   \ 
   p_{t}\geq 0, \ \forall t.
   \label{Pd_GS'}
   \\
   & x_t\in\{0,1\}, \ \forall t.  \label{Px_GS'}
		\end{align}
	\end{subequations}
The above problem is nontrivial to solve due to the numerical Marcum Q-function $Q_1$. 
The APO algorithm that directly conducts joint optimization over $\{\mathcal{P},\mathcal{X}\}$ becomes inapplicable.
Thus, we decompose $\mathsf{P}_{\mathrm{GS}}'$ into a two-layer optimization problem, with the outer iteration solving $\mathcal{X}$'s subproblem and the inner iteration solving $\mathcal{P}$'s subproblem.

The $\mathcal{X}$'s subproblem is solved via iterative local search \cite{neumann2007randomized}, which is guaranteed to achieve a local optimal solution. 
First, we start from the APO solution to $\mathsf{P}_{\mathrm{GS}}$, denoted as $\mathbf{x}^{[0]}$.
This $\mathbf{x}^{[0]}$ is over-confident and may lead to violation of constraints \eqref{Pc_GS'} for some $t$, thus failure of image uploads.
To this end, we define the neighborhood of $\mathbf{x}^{[0]}$ as:
\begin{align}
	\mathcal{N}(\mathbf{x}^{[0]})=\{\mathbf{x}:||\mathbf{x}-\mathbf{x}^{[0]}||_0\leq \eta,~\mathbf{x}\in\{0,1\}^T\},
\end{align}
which $1 \leq \eta \leq T$ is the variable size. 
Second, we randomly flip $\eta$ elements of $\mathbf{x}^{[0]}$, to generate a new feasible solution $\mathbf{x}' \in \mathcal{N}(\mathbf{x}^{[0]})$. Third, with the choice of $\mathbf{x}$ fixed to $\mathbf{x}=\mathbf{x}'$, we check the feasibility of $\mathsf{P}_{\mathrm{GS}}'$ by solving the $\mathcal{P}$'s subproblem, and obtain the associated feasibility flag $\mathsf{feasible}$. 
Fourth, we consider the following two branches: 
\begin{itemize}
    \item[(i)] If $\mathsf{feasible}=1$ and 
$\frac{1}{T}\sum_{t=1}^T\Theta_t(x_t') \leq \frac{1}{T}\sum_{t=1}^T\Theta_t(x_t^{[0]}) $, we update $\mathbf{x}^{[1]}\leftarrow\mathbf{x}'$ and construct the neighborhood of $\mathcal{N}(\mathbf{x}^{[1]})$ to execute the next iteration; 
\item[(ii)]  Otherwise, we execute the random flipping process repeatedly to generate another feasible solution within the $\mathcal{N}(\mathbf{x}^{[0]})$ until 
condition (i) is satisfied. 
\end{itemize}
This process is executed iteratively until the maximum number of iterations $\overline{\mathcal{I}}$ is reached, generating a sequence $\{\mathbf{x}^{[1]},\mathbf{x}^{[2]},\cdots\}$.

Now, the remaining problem is how to solve the $\mathcal{P}$'s subproblem given fixed $\mathbf{x}=\mathbf{x}'$.
An efficient way to provide a feasibility check of $\mathsf{P}_{\mathrm{GS}}'$ given $\mathbf{x}=\mathbf{x}'$ is to first minimize the total transmit power via the following problem
\begin{subequations}
		\label{F}
		\begin{align}
			\mathsf{F}_{\mathrm{GS}}:\,\,\,
			\min_{\mathcal{P}
			}~&\frac{1}{T} \sum_{t=1}^{T}p_{t}  \\
			\textrm{s.t.} ~~ 
&\Gamma(x_t',p_t)\leq \epsilon, \ p_{t}\geq 0, \ \forall t,
   \label{Fb_GS}
		\end{align}
	\end{subequations}
and then check whether the optimal $\mathcal{P}^*$ to $\mathsf{F}_{\mathrm{GS}}$ satisfies $\frac{1}{T} \sum_{t=1}^{T}p_{t}^*\leq P$.
If so, problem $\mathsf{P}_{\mathrm{GS}}'$ given $\mathbf{x}=\mathbf{x}'$ is feasible;
otherwise, the transmit power budget cannot support MR message exchanges and $\mathsf{P}_{\mathrm{GS}}'$ given $\mathbf{x}=\mathbf{x}'$ is infeasible.

\begin{algorithm}[!t]
    \caption{Robust GSCLO via BILS.}
        \begin{algorithmic}[1]
            \State \textbf{Initialize} $\mathbf{x}^{[0]}=\mathbf{x}^{\mathrm{APO}}$ and $\eta=5$. 
            \State Set $i=0$ and iteration counter $n=0$.
            \State \textbf{Repeat}
            \State \ \ \ Sample a solution $\mathbf{x}'\in\mathcal{N}(\mathbf{x}^{[i]})$.
            \State \ \ \ Solve $\mathsf{F}_{\mathrm{GS}}$ using bisection and obtain solution $\{p_{t}^*\}$.
            \State \ \ \ If $\frac{1}{T} \sum_{t}p_{t}^*\leq P$ \& 
$\frac{1}{T}\sum_{t}\Theta_t(x_t') \leq \frac{1}{T}\sum_{t}\Theta_t(x_t^{[i]})$:
            \State \ \ \ \ \ \ $\mathbf{x}^{[i+1]}\leftarrow\mathbf{x}'$, $\mathbf{p}^{[i+1]}\leftarrow\mathbf{p}^*$, $i\leftarrow i+1$.
            \State \ \ \ $n\leftarrow n+1$.
            \State \textbf{Until} $n=\overline{\mathcal{I}}$ and set $\{\mathbf{x}^\diamond=\mathbf{x}^{[i]},\mathbf{p}^\diamond=\mathbf{p}^{[i]}\}$
            \State Output $\{\mathbf{x}^\diamond,\mathbf{p}^\diamond\}$.
        \end{algorithmic}
\end{algorithm}

The benefit of transforming $\mathsf{P}_{\mathrm{GS}}'$ into $\mathsf{F}_{\mathrm{GS}}$ is that power variables $\{p_t\}$ at different frames are decoupled in $\mathsf{F}_{\mathrm{GS}}$, which can be optimized in parallel. 
Specifically, minimizing $\frac{1}{T} \sum_{t=1}^{T}p_{t}$ is equivalent to minimizing $p_t$ for all $t$.
For each $p_t$, it must be bounded by $0 \leq p_t \leq P$. 
Moreover, the marcum Q function $Q_1$ is a monotonically increasing function of $p_t$. 
Consequently, we can conduct a bisection search of $p_t$ within interval $[0,P]$, so as to find the minimum $p_t^*$ satisfying constraint \eqref{Fb_GS}. 
The complexity of bisection for solving $\mathsf{F}_{\mathrm{GS}}$ is $\mathcal{O}(T\mathrm{log}(P))$. 
The entire bisection-in-the-loop iterated local search (BILS) algorithm is summarized in Algorithm 2.

To execute the proposed Algorithm 1 (i.e., APO GSCLO) and Algorithm 2 (i.e., Robust GSCLO), we only need to know the large-scale fading (i.e., channel gain) as seen from problems $\mathsf{P}_{\mathrm{GS}}$ and $\mathsf{P}_{\mathrm{GS}}'$. Given the server position, channel gains can be pre-determined by a radio map $\mathsf{RM}$, which is a function that maps
the robot location $(a_t,b_t)$ to the corresponding signal attenuation as $|h_{t}|^2=
\mathsf{RM}(a_t,b_t)$.
This radio map can be obtained by learning from radio measurements in various positions \cite{levie2021radiounet}. 
The variance of radio measurements can be adopted to compute channel uncertainties $\omega^2$ \cite{malmirchegini2012spatial}.
Note that to reduce the cost of radio measurements, it is possible to combine ray tracing \cite{remley2000improving}, e.g., Nvidia sionna (https://developer.nvidia.com/sionna) with the measurement dataset for faster radio mapping.

\section{More Extensions of GSCLO}\label{section6}

This section presents more variants of GSCLO to achieve low-power GSMR and multi-robot GSMR.

\subsection{Low-Power GSCLO}

The nominal problem $\mathsf{P}_{\mathrm{GS}}$ minimizes the GSMR loss under a given power budget, which is often adopted in resource-constrained systems. 
In QoE constrained systems, it is necessary to minimize the total power consumption under RoboMR QoE constraints, which leads to the following related problem of $\mathsf{P}_{\mathrm{GS}}$:
	\begin{subequations}
		\label{P2}
		\begin{align}
			\mathsf{Q}_{\mathrm{GS}}:\,\,\,
			\min_{\mathcal{P},\mathcal{X}
			}~&\frac{1}{T} \sum_{t=1}^{T}p_{t} \label{Pa'}  \\
			\textrm{s.t.} ~~ &
   \frac{1}{T}\sum_{t=1}^T\Theta_t(x_t) \leq L_{\mathrm{th}},
   \\
   &\textsf{constraints }(\ref{Pc_GS}), (\ref{Px_GS}), \ 
   p_{t}\geq 0, \ \forall t,
		\end{align}
	\end{subequations}
where $L_{\mathrm{th}}$ is a threshold to guarantee desired QoE. 
According to \cite{kerbl20233d} and our own experiments, to ensure satisfactory image qualities, the PSNR should be controlled between $30$\,dB to $40$\,dB, which corresponds to a loss between $0.02$ and $0.04$. 
Therefore, we can set $L_{\mathrm{th}}\in[0.02,0.04]$.
The above problem $\mathsf{Q}_{\mathrm{GS}}$ can be solved by following a similar APO approach in Section IV. 

Problem $\mathsf{Q}_{\mathrm{GS}}$ guarantees the average QoE over all images. In practice, one may also consider individual QoE constraints for each image frame, which leads to 
\begin{subequations}
		\label{P3}
		\begin{align}
			\mathsf{Q}_{\mathrm{GS}}':\,\,\,
			\min_{\mathcal{P},\mathcal{X}
			}~&\frac{1}{T} \sum_{t=1}^{T}p_{t} \label{Pa''}  \\
			\textrm{s.t.} ~~ &
   \Theta_t(x_t) \leq L_{\mathrm{th}}, \forall t
   \\
   &\textsf{constraints }(\ref{Pc_GS}), (\ref{Px_GS}), \ 
   p_{t}\geq 0, \ \forall t. 
		\end{align}
	\end{subequations}
This problem has a closed-form solution, which is given in the following proposition. 
\begin{proposition}
The optimal solution $\{x_t^\star,p_t^\star\}$ to $\mathsf{Q}_{\mathrm{GS}}'$ is
\begin{equation}\label{prop3}
(x_t^*,p_t^*) = 
\left\{\begin{array}{ll}
(1,\frac{\sigma^2}{|h_{t}|^2}
2^{
\frac{I}{\tau B}
}-\frac{\sigma^2}{|h_{t}|^2}), \ &\mathrm{if}~L_t>L_{\mathrm{th}} \\
(0,\frac{\sigma^2}{|h_{t}|^2}
2^{
\frac{S}{\tau B}
}-\frac{\sigma^2}{|h_{t}|^2}), \ &\mathrm{if}~L_t \leq L_{\mathrm{th}}
\end{array}\right.,
\end{equation} 
\end{proposition}
\begin{proof}
We first observe that the optimal solution must activate constraint \eqref{Pc_GS}, which gives 
\begin{align}
    p_{t}=\frac{\sigma^2\left[2^{\frac{(I-S)x_t + S}{\tau B}}-1\right]}{|h_{t}|^2}.
\end{align}
Substituting this equation into the cost function of $\mathsf{Q}_{\mathrm{GS}}'$, the problem is equivalently rewritten as 
	\begin{subequations}
		\begin{align}
			\min_{\mathcal{X}
			}~&\frac{1}{T} \sum_{t=1}^{T}\frac{\sigma^2\left[2^{\frac{(I-S)x_t + S}{\tau B}}-1\right]}{|h_{t}|^2} \label{Pa'}  \\
			\textrm{s.t.} ~~ &
  L_t(1-x_t) -L_{\mathrm{th}}\leq 0, \ x_t(1-x_t)=0, \ \forall t. \label{22b}
		\end{align}
	\end{subequations}
The objective function is a monotonically increasing function of $x_t$. 
This implies that changing $x_t$ from $1$ to $0$ always reduces the power. 
Combining the constraints \eqref{22b}, we immediately obtain \eqref{prop3}, and the proof is completed.
\end{proof}

According to \textbf{Proposition 3}, the frame transmit power $p_t$ is inversely proportional to the wireless channel gain $|h_{t}|^2$. 
However, it is exponentially dependent on the content data volume $(I,S)$, which is further controlled by GS losses $\{L_t\}$.
The above observations disclose that in GSMR systems, the GS parameters will have more significant impacts on the physical layer design than those of the wireless channels. 
Therefore, the principle of GSMR is to allocate more power to frames associated with high GS losses. 

\textbf{Proposition 3} can be used to evaluate the communication overhead saving brought by GSMR compared to RoboMR. 
Specifically, we compute the power consumption of $\sum p_t^{\mathrm{MR}}$, where $p_t^{\mathrm{MR}}$ is the solution to $\mathsf{Q}_{\mathrm{GS}}$ under condition $\{x_t=1|\forall t\}$. 
This gives 
\begin{align}
&p_t^{\mathrm{MR}}=\frac{\sigma^2(2^{\frac{I}{\tau B}}-1)}{|h_{t}|^2}.
\end{align}
{
Thus, we have
\begin{align}
\Delta P
&=
\sum_{t=1}^Tp_t^{\mathrm{MR}}-\sum_{t=1}^T p_t^*
\nonumber\\
&=
\sum_{t=1}^T\frac{\mathbb{I}_{L_t\leq L_{\mathrm{th}}}(L_t)(2^{\frac{I}{\tau B}}-
2^{\frac{S}{\tau B}}
)}
{|h_{t}|^2}
,
\end{align}
where indicator function $\mathbb{I}_{\mathcal{C}}(x)=1$ if $x\in\mathcal{C}$ and zero otherwise.}
We refer to this quantity as \textbf{GS power saving factor}. 
It can be seen that this factor is proportional to $\sum_{t=1}^T\frac{\mathbb{I}_{L_t\leq L_{\mathrm{th}}}(L_t)}{|h_{t}|^2}$, meaning that the power saving is determined by the percentage of qualified GS images.

{
\subsection{Multi-Robot GSCLO}

In multi-robot settings, the key is to avoid multi-user interference during data uploading. 
A common way is to adopt multi-antenna signal processing at the server \cite{wang2020angle}, and the aggregated signal received at edge server at time $t$ becomes
\begin{equation}
    \mathbf{z}_t = \sum_{k=1}^K x_{k,t}\mathbf{h}_{k,t} s_{k,t}+\mathbf{n}_{t},
    \label{eq:signal}
\end{equation}
where $N$ is the number of antennas, $K$ is the number of robots, $\mathbf{z}_t=[z_{1,t},\cdots,z_{N,t}]^T\in\mathbb{C}^{N\times 1}$ is the received signal, $\mathbf{h}_{k,t} \in \mathbb{C}^{N \times 1}$ denotes the channel from the $k$-th robot to the edge server, and $\mathbf{n}_{t}\in\mathbb{C}^{N\times 1}$ is noise vector with zero mean and covariance $\mathbb{E}\{\mathbf{n}_t\mathbf{n}_t^{H}\}=\sigma^2\mathbf{I}_{N}$ ($\mathbf{I}_{N}$ is the identity matrix of size $N\times N$). 
Let $\mathbf{H}_{t}=[\mathbf{h}_{1,t},\ldots,\mathbf{h}_{K,t}]$, $\forall k$.
By applying the zero-forcing receiver \cite{wang2020angle} $[\mathbf{w}_{1,t}^{H},\mathbf{w}_{2,t}^{H},...,\mathbf{w}_{K,t}^{H}]^{T}=(\mathbf{H}_{t}^{H}\mathbf{H}_{t})^{-1}\mathbf{H}_{t}^{H}$ to recover the uplink signal, i.e., $\widehat{z}_{k,t}\!=\!\mathbf{w}_{k,t}^{H}\mathbf{z}_{t}\!$, we have $\mathbf{w}_{i,t}^{H}\mathbf{h}_{j,t}=0$, $\forall i \neq j$. 
Therefore, the data-rate of robot $k$ is
\begin{align}
&R_{k,t}(p_{k,t})=B\mathrm{log}_2\left(1+\frac{H_{k,t}p_{k,t}}
	{\sigma^2}\right).
\end{align}
where 
\begin{align}
H_{k,t}=\frac{|\mathbf{w}_{k,t}^{H}\mathbf{h}_{k,t}|^{2}}{\|\mathbf{w}_{k,t}^{H}\|_2^{2}}.
\end{align}
Accordingly, problem $\mathsf{P}_{\mathrm{GS}}$ in multi-robot settings is formulated as 
	\begin{subequations}
		\label{P2}
		\begin{align}
			&\mathsf{M}_{\mathrm{GS}}:\,\,\,
			\min_{\{p_{k,t},x_{k,t}\}
			}~\frac{1}{T}\sum_{t=1}^T\sum_{k=1}^KL_{k,t}(1-x_{k,t})  \\
			\textrm{s.t.} ~~ & \tau R_{k,t}(p_{k,t})
            \geq x_{k,t} I + (1-x_{k,t})S,\ \forall k,t,
 \\  
     &\frac{1}{T} \sum_{t=1}^{T}p_{k,t} \leq P, 
   \ 
   p_{k,t}\geq 0, \ \forall k,t,
   \\
   & x_{k,t}\in\{0,1\}, \ \forall k,t.
		\end{align}
	\end{subequations}
Problem $\mathsf{M}_{\mathrm{GS}}$ has the same structure as problem $\mathsf{P}_{\mathrm{GS}}$, and can be readily solved using the APO algorithm.
}

\section{Experiments}\label{section7}

We implement the RoboMR system exploiting C++ and Python in ROS. 
The real $4\,\text{m}\times 4\,\text{m}$ Agilex (AGX) sandbox platform is shown in Fig.~\ref{platforms}a.
The virtual world constructed using the CARLA simulator \cite{carla} is shown in Fig.~\ref{platforms}b, where two additional virtual objects (marked in white boxes) are added to the virtual world. 
The CARLA simulator is implemented on a workstation with a $3.7$\,GHZ AMD Ryzen 9 5900X CPU and an NVIDIA $3090$\,Ti GPU.

As shown in Fig.~\ref{platforms}c, we adopt a robot named LIMO, which has a 2D lidar, an RGBD camera, and an onboard NVIDIA Jetson Nano computing platform for executing the localization and navigation packages. 
To realize RoboMR, the robot is connected to the CARLA simulator via ROS bridge \cite{ros-bridge}.
The robot route is marked as the blue line in Fig.~\ref{platforms}d.
We navigate the LIMO robot along the route for two rounds. 
In the first round, the robot collects $H=288$ frames (including images and poses) for training a GS model \cite{kerbl20233d}. 
In the second round, the robot collects $T=288$ frames for evaluation. 
Each image has a data volume of $I=537.6$\,Kbits (i.e., $67.2$\,KBytes). 
The robot pose has $6$ floating-point numbers, with $S=192$\,bits.

\begin{figure}[!t]
	\centering
	\begin{subfigure}{0.45\linewidth}
		\centering
		\includegraphics[width=\linewidth]{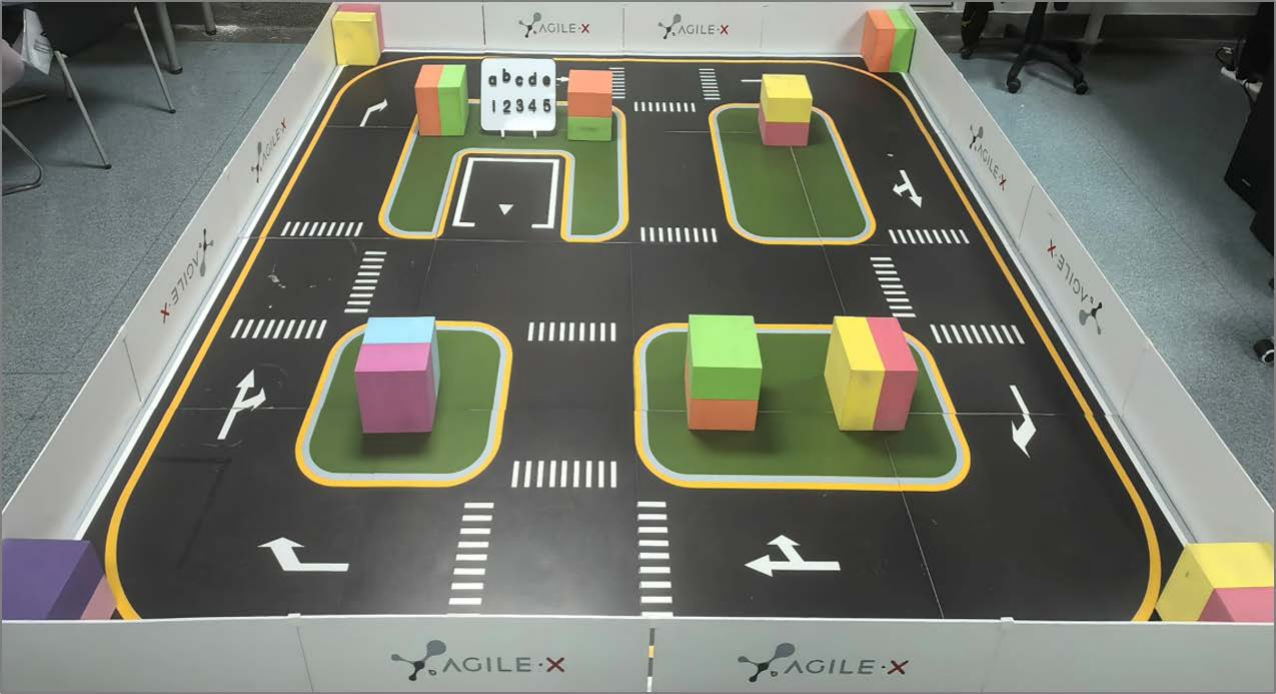}
		\caption{Real AGX.}
	\end{subfigure}
    	\begin{subfigure}{0.45\linewidth}
		\centering
		\includegraphics[width=\linewidth]{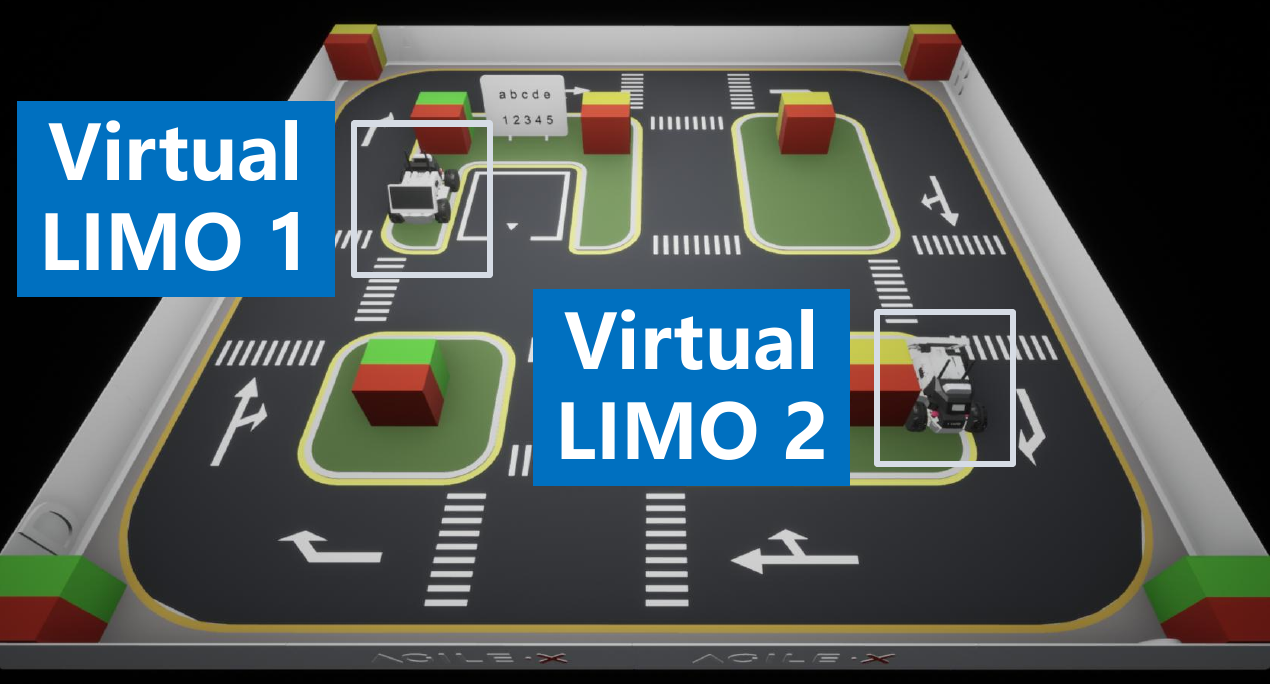}
		\caption{Virtual AGX.}
	\end{subfigure}
    	\begin{subfigure}{0.4\linewidth}
		\centering
		\includegraphics[width=\linewidth]{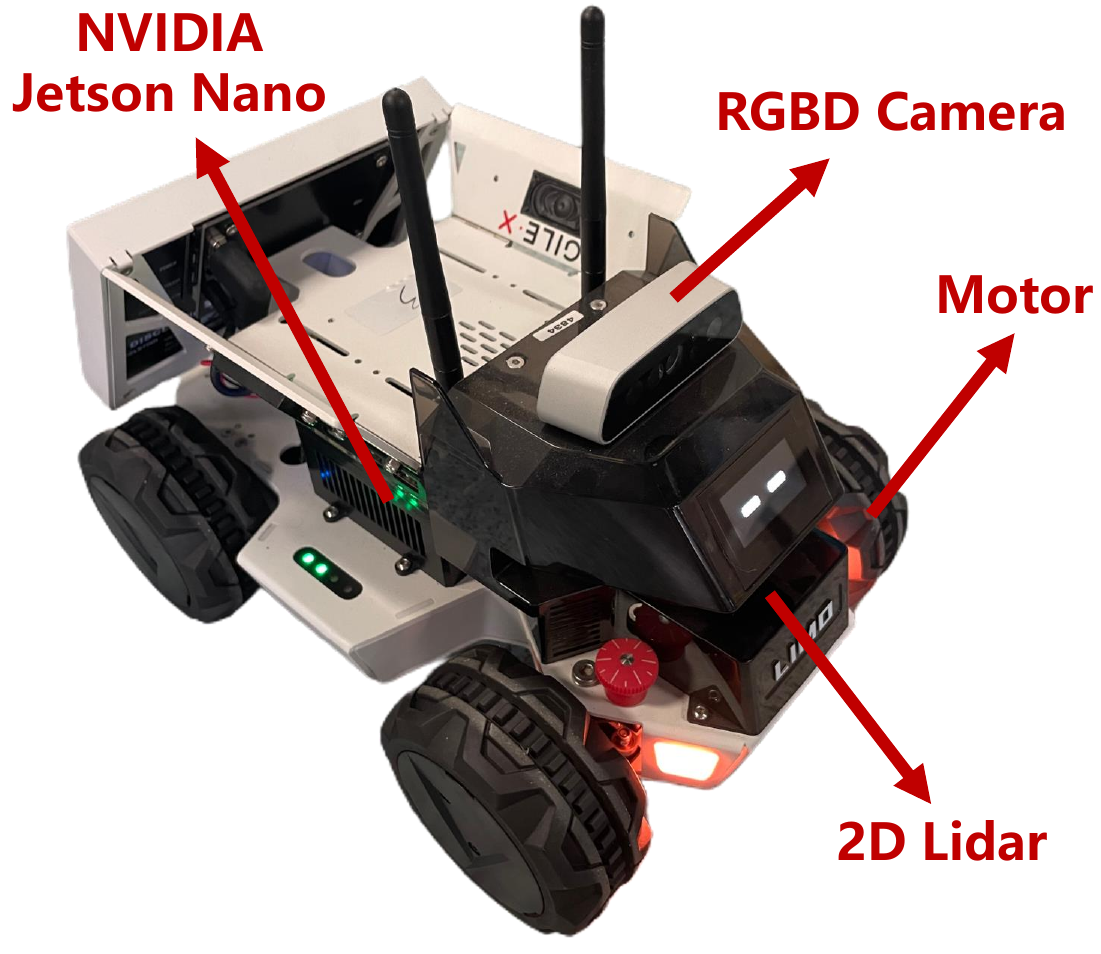}
		\caption{LIMO robot.}
	\end{subfigure}
    \begin{subfigure}{0.35\linewidth}
		\centering
		\includegraphics[width=\linewidth]{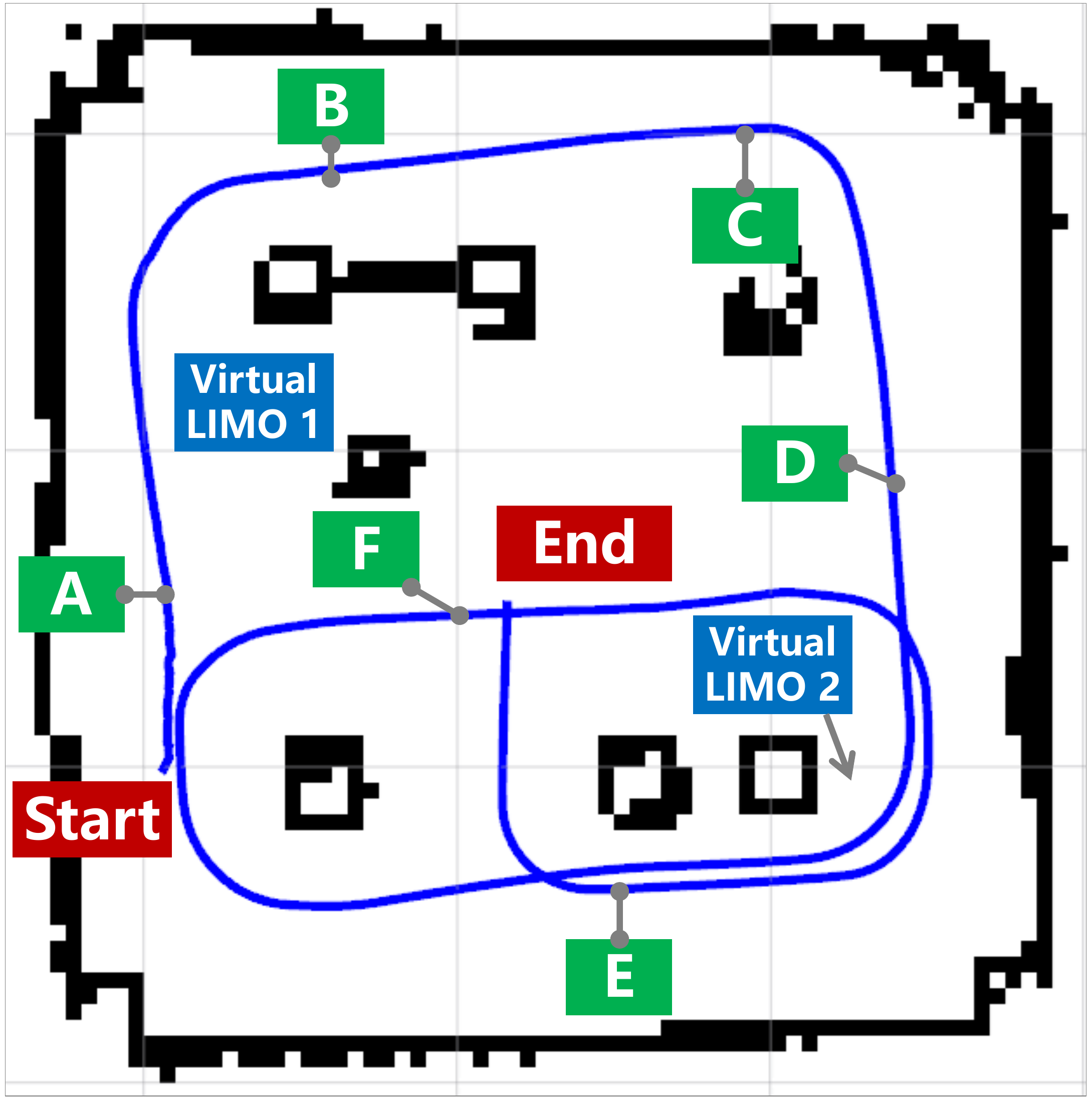}
		\caption{Robot route.}
	\end{subfigure}
	\caption{Implementation of the RoboMR platform. In (d), the blue line represents robot trajectory, red boxes represent start and end positions, green boxes represent image poses, and blue boxes represent positions of virtual agents, respectively.}
	\label{platforms}
    
\end{figure}

\begin{figure}[!t]
	\centering
	\begin{subfigure}{0.44\linewidth}
		\centering
		\includegraphics[width=\linewidth]{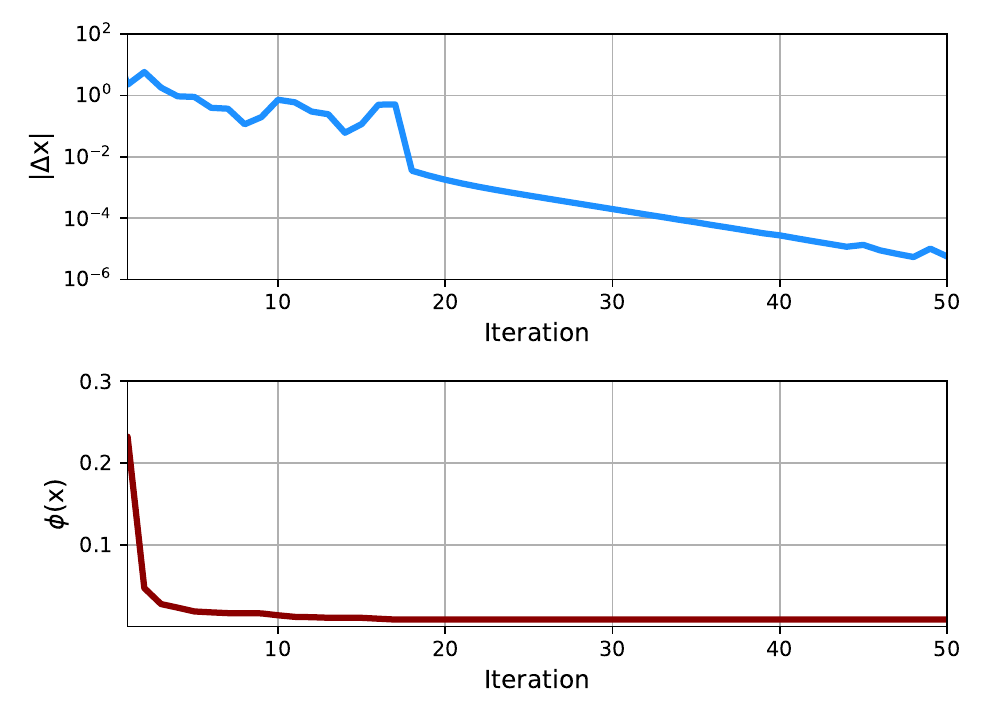}
		\caption{APO.}
	\end{subfigure}
 	\begin{subfigure}{0.52\linewidth}
		\centering
		\includegraphics[width=\linewidth]{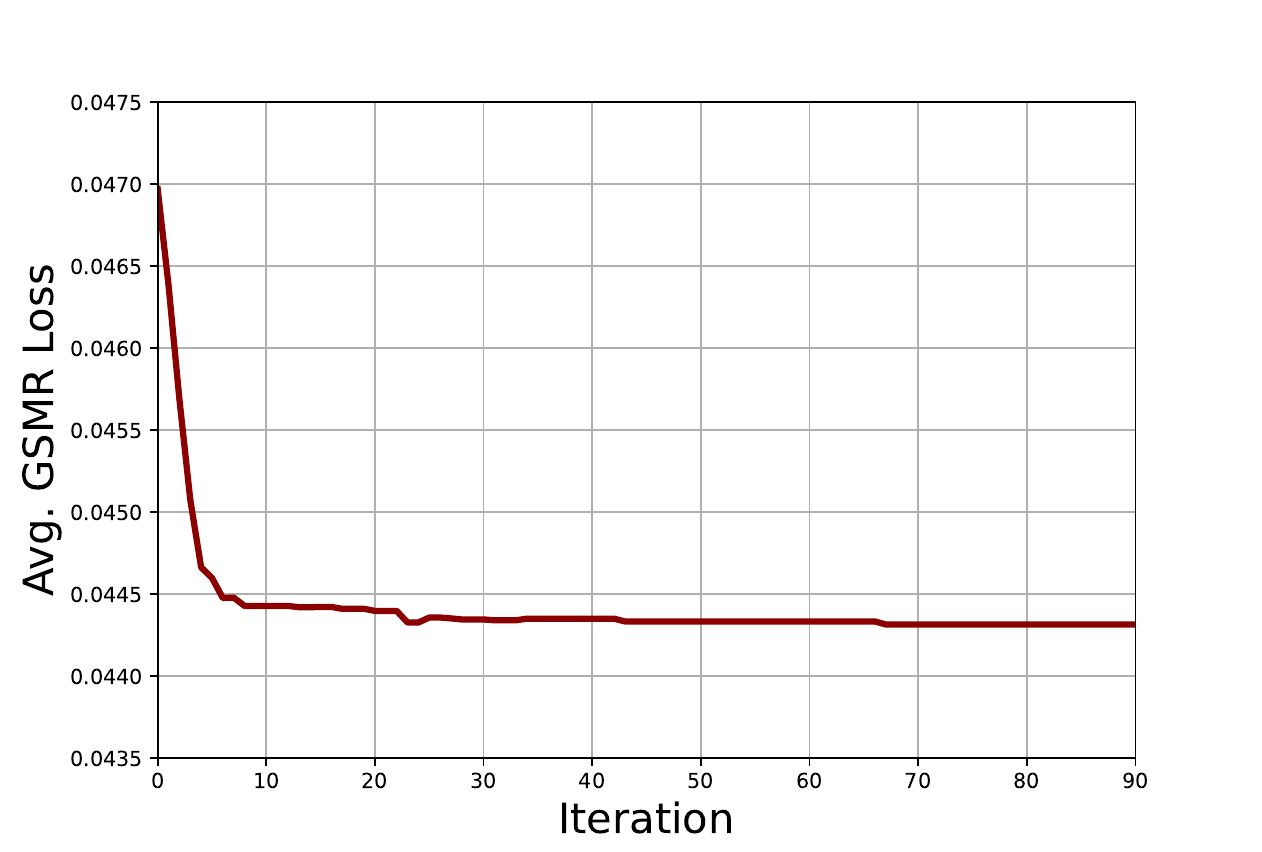}
		\caption{BILS.}
	\end{subfigure}
	\caption{{Convergence analysis. (a) $\|\Delta\mathbf{x}\|$ and $\phi(\mathbf{x})$ versus $n$ for APO; (b) GSMR loss versus $n$ for BILS.}}
	\label{fig:converge}
\end{figure}

\begin{figure}[!t]
	\centering
	\begin{subfigure}{0.48\linewidth}
		\centering
		\includegraphics[width=\linewidth]{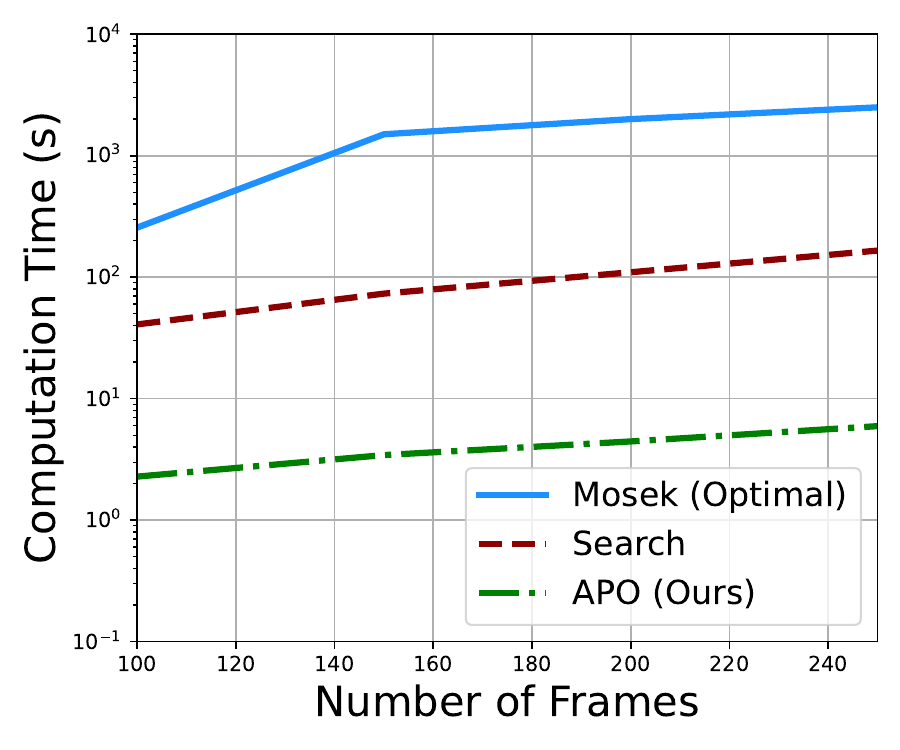}
		\caption{Computation time.}
	\end{subfigure}
 	\begin{subfigure}{0.48\linewidth}
		\centering
		\includegraphics[width=\linewidth]{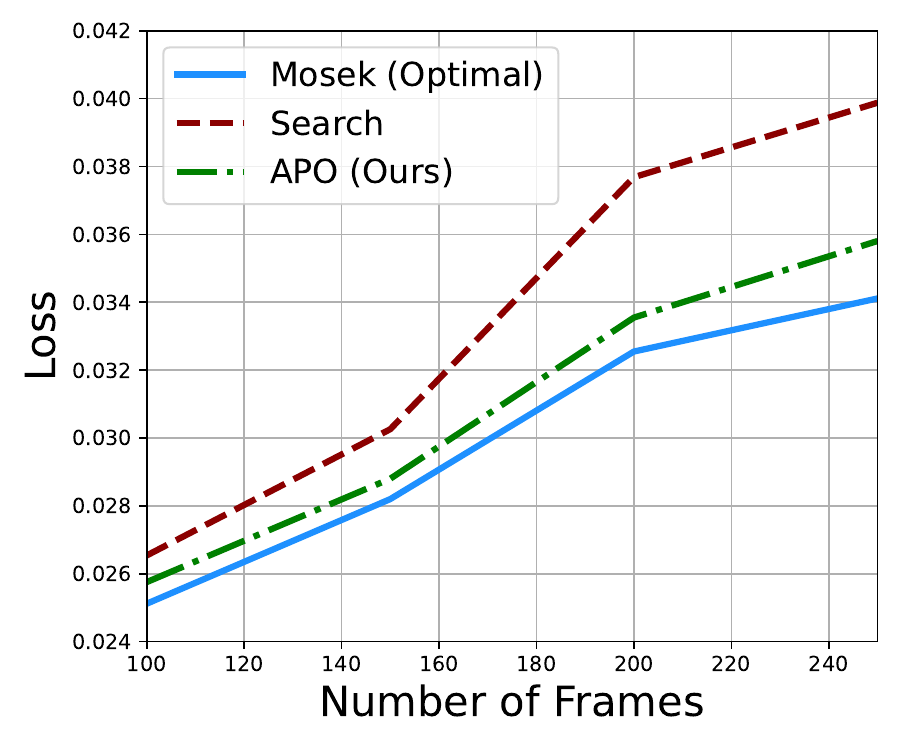}
		\caption{Loss comparison.}
	\end{subfigure}
	\caption{Comparison of computation time and GSMR loss.}
	\label{compare1}
\end{figure}

Without otherwise specified, the time step $\tau=0.1$\,s, bandwidth $B=1$\,MHz, noise power $\sigma^2=-60$\,dBm \cite{wang2020angle}.
The channel is assumed to be Rician fading, i.e., \cite{wang2020angle}
\begin{align}\label{ht}
&h_t=\sqrt{\varrho_0 d_t^{-\alpha}}\Big(\sqrt{\frac{\mathcal{K}}{1+\mathcal{K}}}\,g_t^{\mathrm{LOS}}
+\sqrt{\frac{1}{1+\mathcal{K}}}\,g_t^{\mathrm{NLOS}}\Big),
\end{align}
where $\varrho_0=-30$\,dB is the pathloss at $1\,\mathrm{m}$, $d_t$ is the robot-server distance (i.e., $10$\,m), and $\alpha=3$ is the pathloss exponent.
Notice that $\mathcal{K}$ is the Rician K-factor accounting for propagation effects of the ling-of-sight (LoS) and non-LoS links.
The LoS component is $g_t^{\mathrm{LOS}}=\mathrm{exp}\left(-\mathrm{j}\pi\,\mathrm{sin}\,\psi_t\right)$ with $\psi_t\in\mathcal{U}(-\pi,\pi)$ being the phase angle, and the non-LoS component is $g_t^{\mathrm{NLOS}}\sim\mathcal{CN}(0,1)$ \cite{wang2020angle}.
{All quantitative results are obtained by averaging over $50$ random simulation runs, with independent channels in each run. }

We compare GSMR with APO to the following baselines:
\begin{itemize}
   \item[1)] \textbf{RoboMR}: Vanilla RoboMR \cite{li2024seamless} uploading all images (i.e., $\{x_t=1\}$);
   \item[2)] \textbf{RoboGS}: Vanilla GS \cite{kerbl20233d} with no image uploading (i.e., $\{x_t=0\}$);
   \item[3)] \textbf{MaxRate}: GSMR with water-filling power allocation for sum-rate maximization \cite{zhang2024efficient};
   \item[4)] \textbf{Fairness}: GSMR with max-min fairness power allocation \cite{zheng2016wireless};
   \item[5)] \textbf{Ranking}: GSMR with \eqref{ranking};
   \item[6)] \textbf{Search}: GSMR by solving $\mathsf{P}_{\mathrm{GS}}$ with iterative local search \cite{neumann2007randomized};
   \item[7)] \textbf{Rounding}: GSMR by solving $\mathsf{P}_{\mathrm{GS}}$ with continuous relaxation and rounding \cite{hubner2014rounding};
   \item[8)]\textbf{Mosek}: GSMR by solving $\mathsf{P}_{\mathrm{GS}}$ with Mosek \cite{diamond2016cvxpy}.
\end{itemize}

\begin{figure*}[!t]
	\centering
	\begin{subfigure}{0.24\linewidth}
		\centering
		\includegraphics[width=1\linewidth]{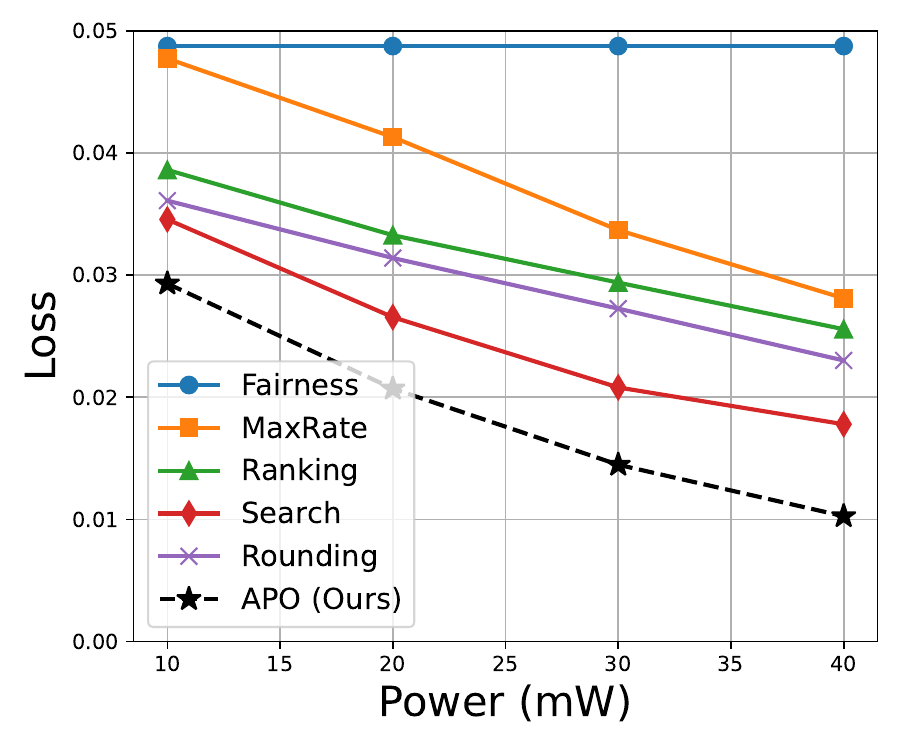}
		\caption{Loss versus $P$.}
	\end{subfigure}
 	\begin{subfigure}{0.24\linewidth}
		\centering
		\includegraphics[width=1\linewidth]{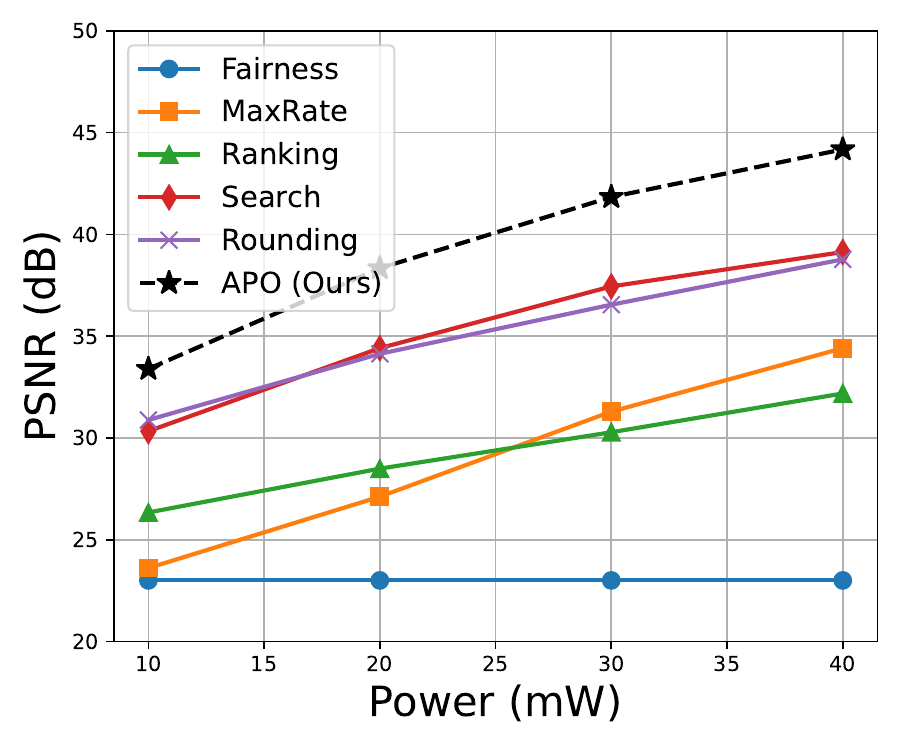}
		\caption{PSNR versus $P$.}
	\end{subfigure}
     	\begin{subfigure}{0.24\linewidth}
		\centering
		\includegraphics[width=1\linewidth]{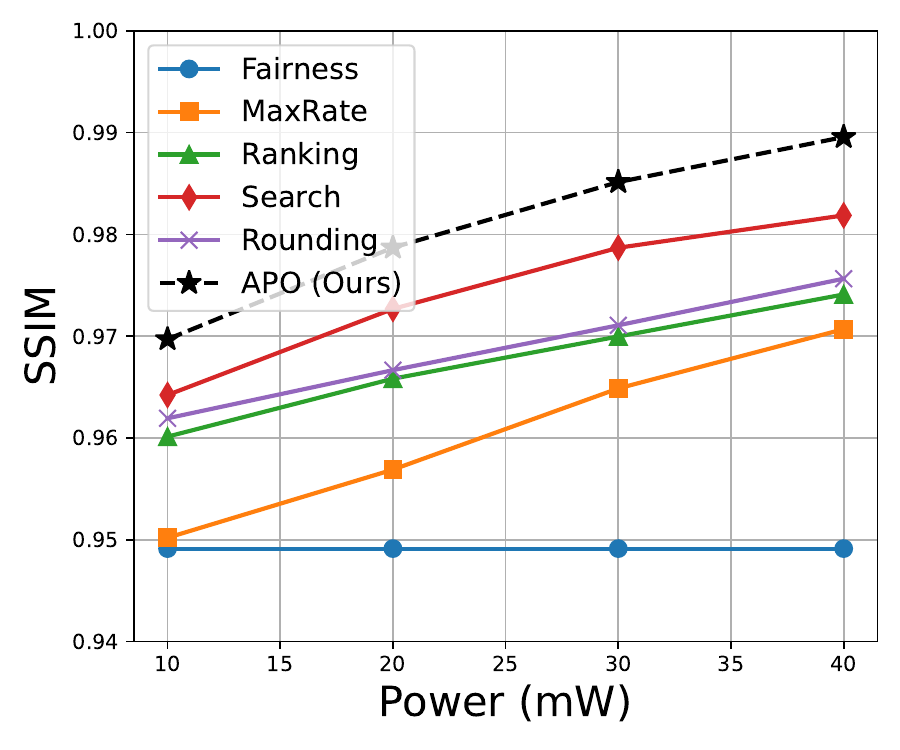}
		\caption{SSIM versus $P$.}
	\end{subfigure}
         	\begin{subfigure}{0.24\linewidth}
		\centering
		\includegraphics[width=1\linewidth]{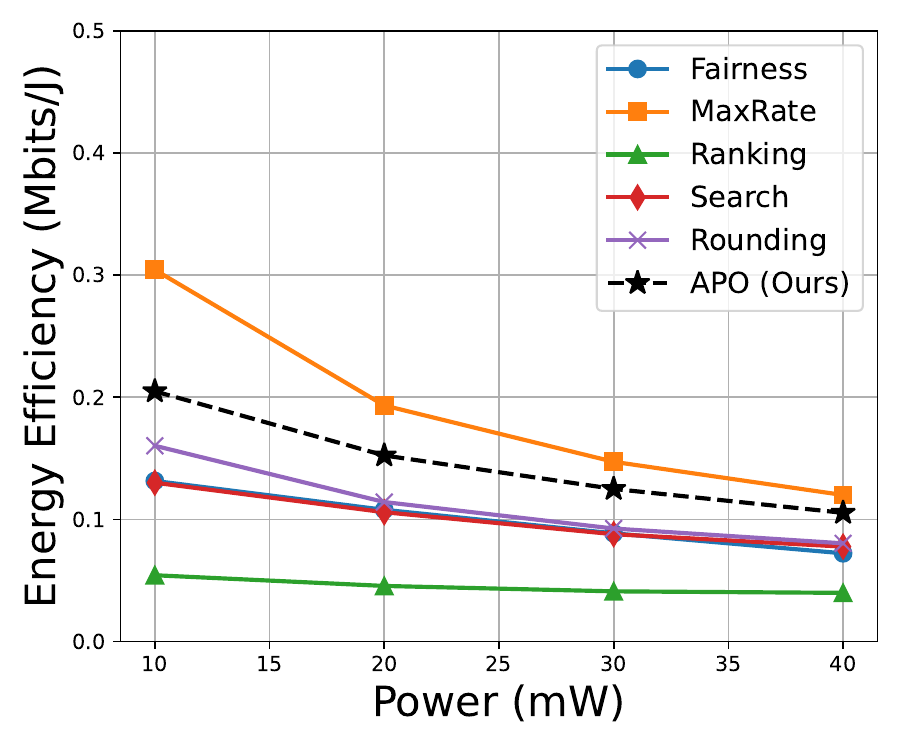}
		\caption{Energy efficiency versus $P$.}
	\end{subfigure}
	\caption{Comparison of loss, PSNR, SSIM, and energy efficiency.}
	\label{compare2}
    \vspace{-0.1in}
\end{figure*}

\begin{figure}[t]
    \centering
    \includegraphics[width=0.49\textwidth]{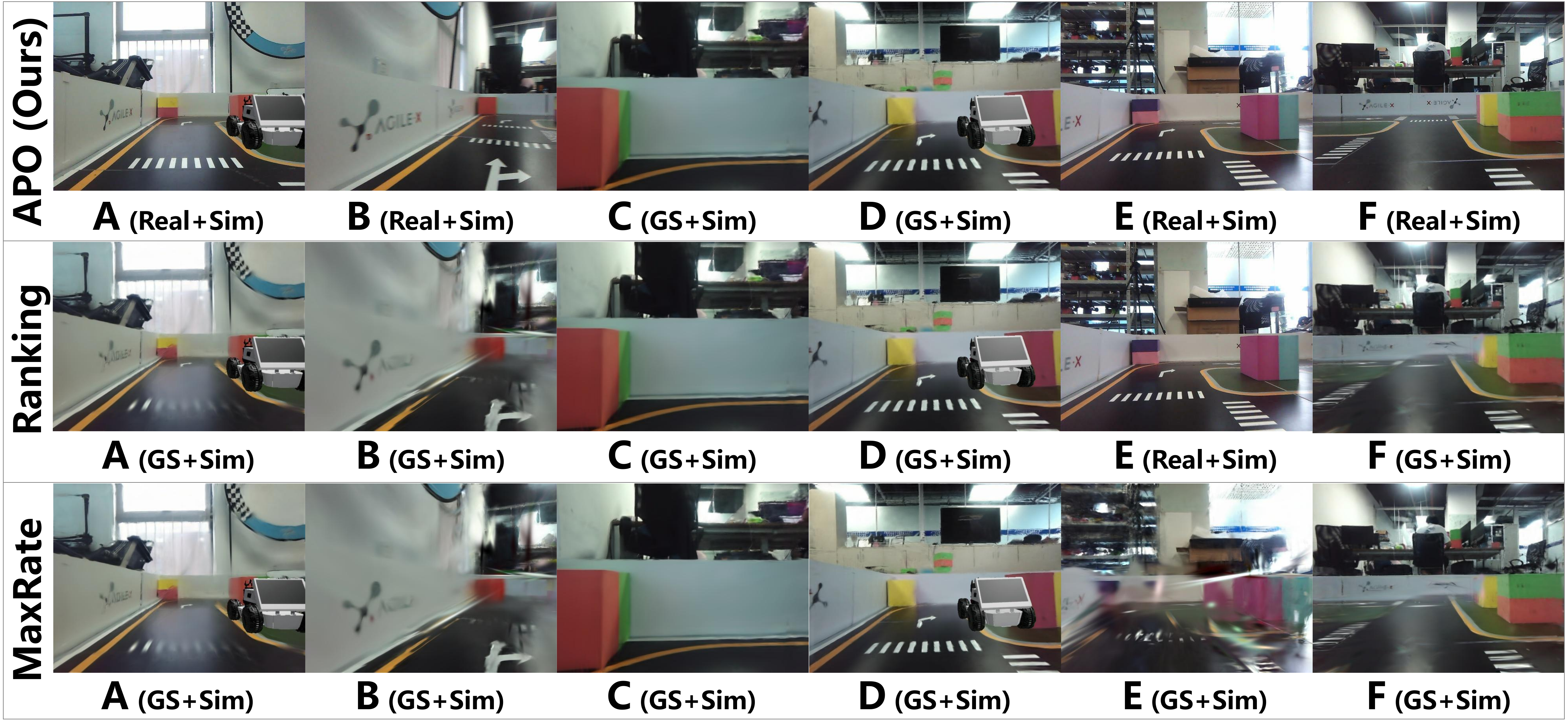}
    \caption{Visualization of six image frames, where their positions are marked in Fig. 4d.}
    \label{fig:demo1}
    \vspace{-0.1in}
\end{figure}

\begin{figure}[!t]
	\centering
	\begin{subfigure}{0.32\linewidth}
		\centering
		\includegraphics[width=\linewidth]{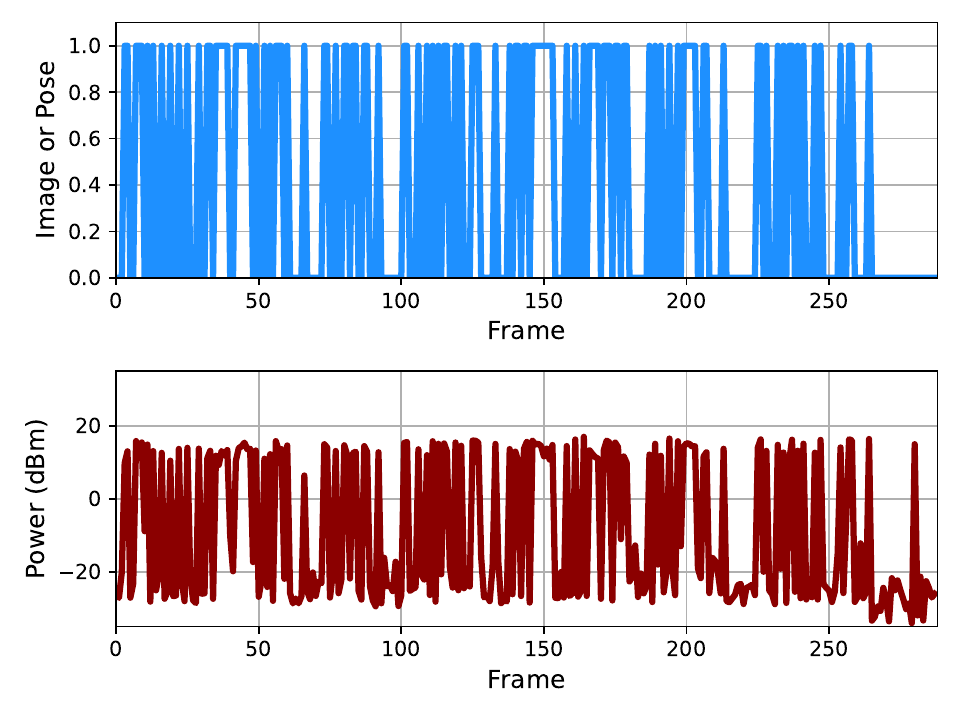}
      
		\caption{APO (Ours).}
	\end{subfigure}
	\centering
	\begin{subfigure}{0.32\linewidth}
		\centering
		\includegraphics[width=\linewidth]{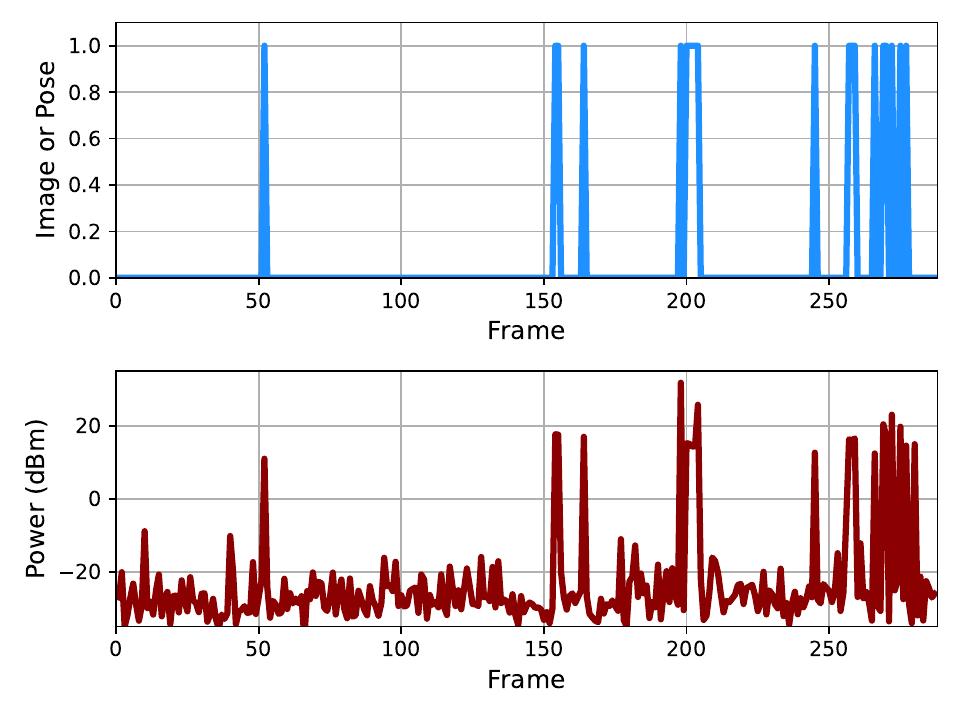}
      
		\caption{Ranking.}
	\end{subfigure}
 	\begin{subfigure}{0.32\linewidth}
		\centering
		\includegraphics[width=\linewidth]{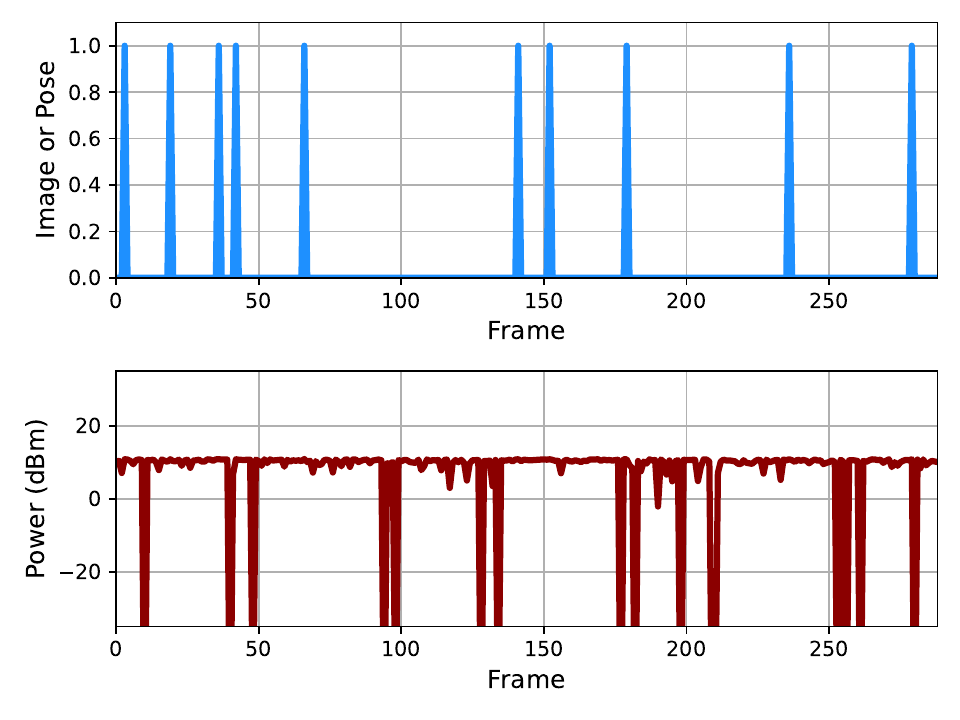}
      
		\caption{MaxRate.}
	\end{subfigure}
	\caption{Content switching $\{x_t\}$ and power profiles $\{p_t\}$.}
     \label{fig:demo2}
     
\end{figure}

\begin{figure}[!t]
	\centering
	\begin{subfigure}{0.49\linewidth}
		\centering
		\includegraphics[width=\linewidth]{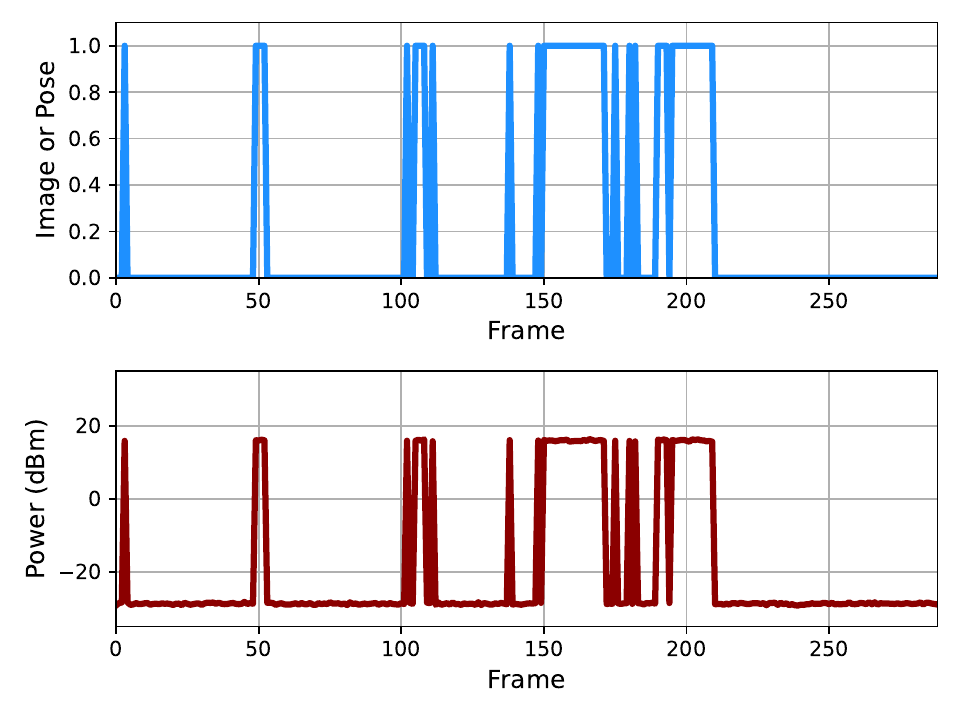}
      
		\caption{APO (Ours).}
	\end{subfigure}
	\centering
	\begin{subfigure}{0.49\linewidth}
		\centering
		\includegraphics[width=\linewidth]{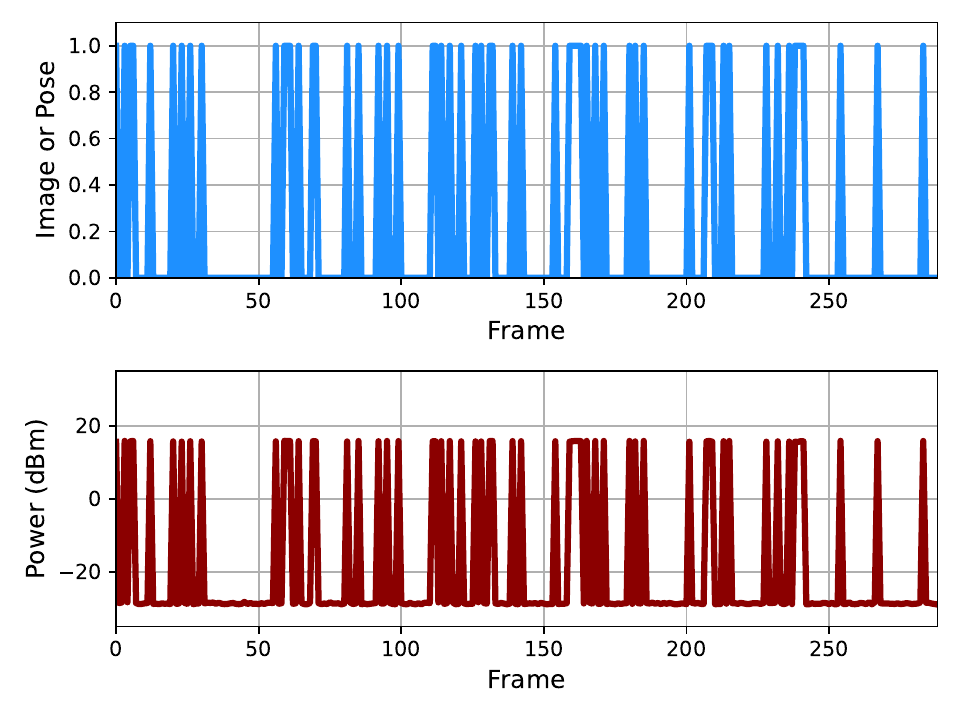}
      
		\caption{MaxImg.}
	\end{subfigure}
	\caption{Comparison between APO and MaxImg.}
     \label{fig:maximg}
     
\end{figure}

\begin{figure*}[!t]
	\centering
	\begin{subfigure}{0.32\linewidth}
		\centering
		\includegraphics[width=1\linewidth]{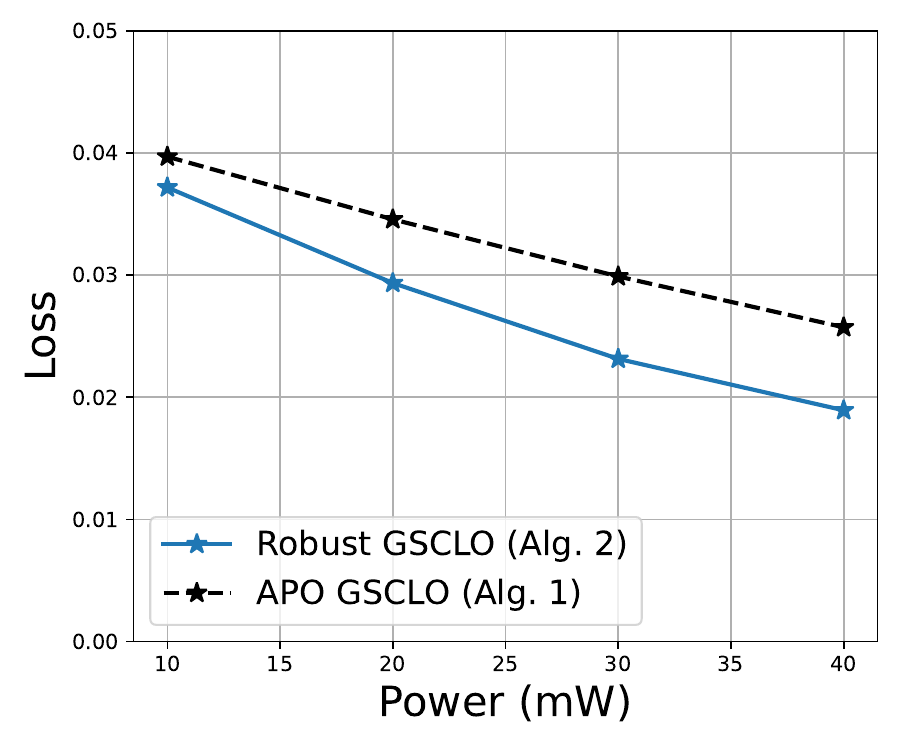}
		\caption{Loss versus $P$.}
	\end{subfigure}
 	\begin{subfigure}{0.32\linewidth}
		\centering
		\includegraphics[width=1\linewidth]{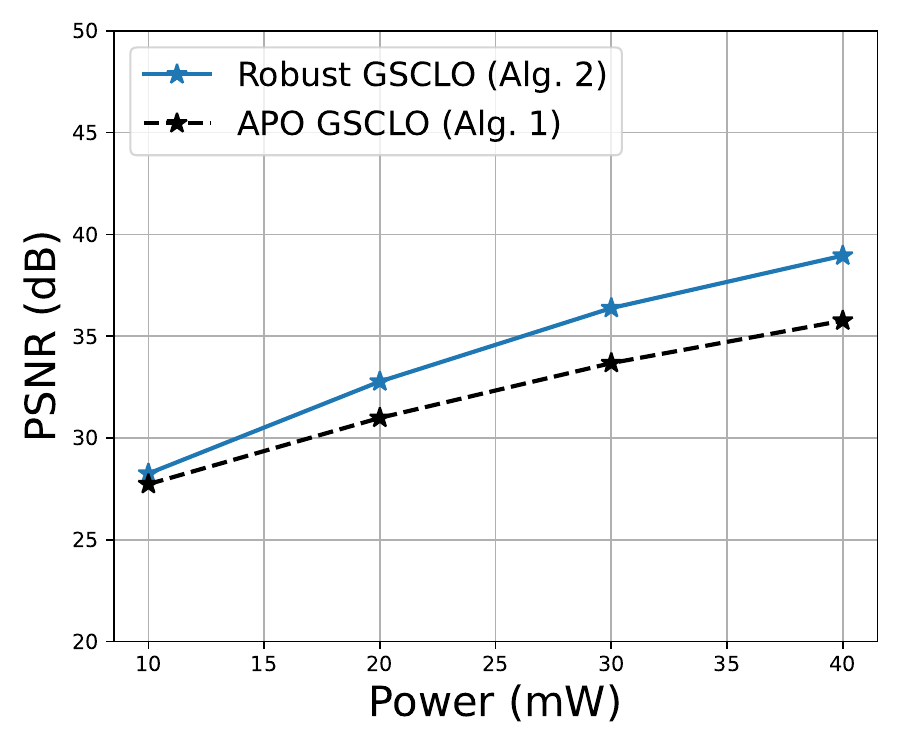}
		\caption{PSNR versus $P$.}
	\end{subfigure}
     \begin{subfigure}{0.32\linewidth}
		\centering
		\includegraphics[width=1\linewidth]{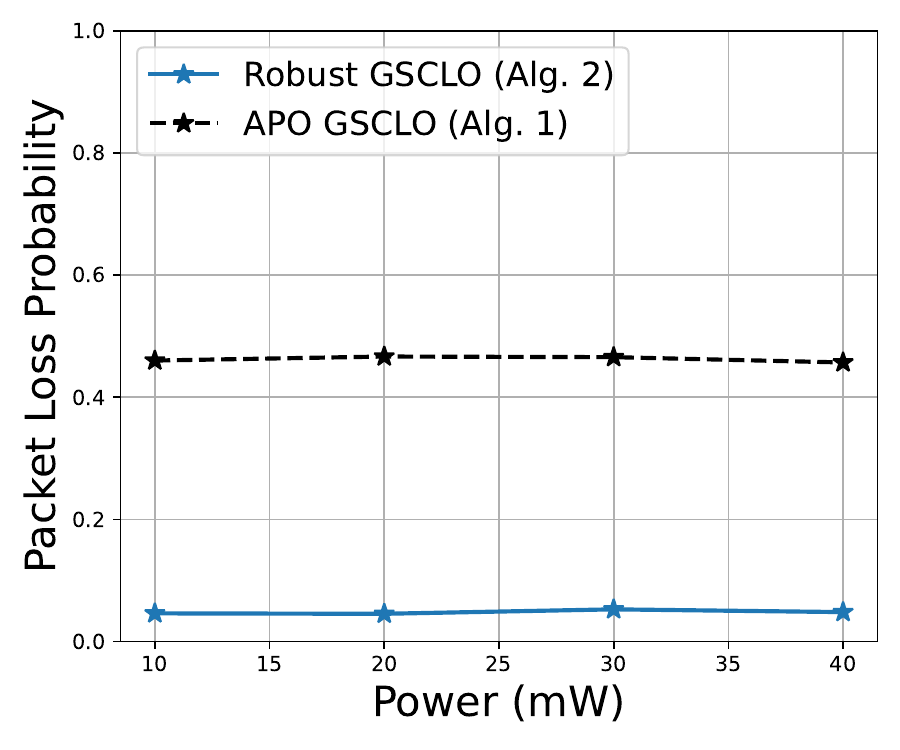}
		\caption{Packet loss probability versus $P$.}
	\end{subfigure}
	\caption{Comparison of loss, PSNR, and packet loss under channel uncertainties.}
	\label{robust}
\end{figure*}

\subsection{Experiment 1: Convergence and Computation Time}

First, we conduct numerical experiments to validate the convergence of the proposed  Algorithm 1 (i.e., APO GSCLO).
Specifically, we consider the case of $P=5$\,mW and $\mathcal{K}=30$\,dB.
The variation between consecutive iterations $\|\Delta\mathbf{x}\|$ with $ \Delta\mathbf{x}=\mathbf{x}^{[n]}-\mathbf{x}^{[n-1]}$ versus the iteration count $n$ is shown in Fig.~\ref{fig:converge}a.
It can be seen that $\|\Delta\mathbf{x}\|$ falls below $10^{-4}$ after $35$ iterations.
This demonstrates the convergence of the proposed APO. 
The zero-one loss $\phi(\mathbf{x})=
\frac{1}{T}\sum_{t=1}^T x_{t}(1-x_t)$ versus the iteration count $n$ is also shown in Fig.~\ref{fig:converge}a. 
It can be seen that $\phi(\mathbf{x})$ is below $0.01$ after $20$ iterations, meaning that the APO algorithm converges to a point close to the desired binary solution.
{On the other hand, the convergence of the Algorithm 2 (i.e., Robust GSCLO) is shown in Fig.~\ref{fig:converge}b. 
It can be seen that the average GSMR loss of Algorithm 2 converges and stabilizes after $50$ iterations.}

Next, we conduct numerical experiments to evaluate the computation time of the proposed APO algorithm.
We consider the same setting as Experiment 1 (i.e, $P=5$\,mW and $\mathcal{K}=30$\,dB).
The computation time and associated MR loss for the Mosek, Search, and APO schemes are shown in Fig.~\ref{compare1}a and Fig.~\ref{compare1}b, respectively. 
The proposed APO reduces the computation time by at least $100$x compared to Mosek in Fig.~\ref{compare1}a, while guaranteeing close-to-optimal performance in Fig.~\ref{compare1}b.
The Search scheme leads to moderate computation time, but the worst MR performance. 
{Note that our method can be further accelerated by multi-core computing, efficient programming languages (e.g., C++), and receding horizon optimization.}

\subsection{Experiment 2: Comparison with Existing Benchmarks}

We compare the proposed GSMR with APO to all the benchmark schemes mentioned in Section VII.
We consider the case of $P=\{10,20,30,40\}$\,mW and $\mathcal{K}=0$\,dB. 
It can be seen from Fig.~\ref{compare2}a that as the power budget increases, the losses of all the simulated schemes are reduced. Moreover, no matter how the power budget varies, the proposed GSMR with APO achieves a loss below $0.03$, and consistently outperforms all the other schemes. 
In particular, the loss reduction is at least $10\%$. 
The PSNR performance of all the simulated schemes are shown in Fig.~\ref{compare2}b.
It can be seen that the proposed APO achieves a PSNR above $30$\,dB under all power budgets. 
Compared with all the other schemes, the improvement of PSNR is at least $5$\,dB, which is significant in RoboMR systems.
By comparing Fig.~\ref{compare2}a and Fig.~\ref{compare2}b, we can find that a PSNR from $30$\,dB to $40$\,dB roughly corresponds to a loss from $0.035$ to $0.02$. 
The SSIM performance of all the simulated schemes are shown in Fig.~\ref{compare2}c.
Again, the proposed GSMR with APO outperforms all the other schemes in terms of the SSIM metric. 
Especially at higher power budgets, our SSIM values are very close to $1$.
This demonstrates the all-round ability of APO.
{Lastly, the energy efficiencies of the simulated schemes are shown in Fig.~\ref{compare2}d.
It can be seen that the proposed APO achieves an energy efficiency close to the optimal bound\footnote{Under a fixed power consumption, the MaxRate scheme would achieve the optimal energy efficiency, since it maximizes the sum-rate using water-filling power allocation \cite{yu2004iterative}}, demonstrating its excellent balance between energy efficiency and MR performance. 
The Ranking scheme only considers the GSMR losses, thus leading to the worst efficiency.
Note that the energy efficiencies of all schemes decease as the power budget increases. 
This is because the transmit power acts as the denominator and the transmission rate increases slower than the transmit power.}

\subsection{Experiment 3: Visualizations}

To demonstrate the benefit brought by cross-layer optimization, we consider the case of $P=10$\,mW and $\mathcal{K}=0$\,dB, and compare APO with the MaxRate and Ranking schemes. The MaxRate represents the communication-centric scheme without considering MR requirements, while the Ranking scheme represents the GS-centric scheme without considering communication constraints.
The qualitative results are shown in Fig.~\ref{fig:demo1}.
It can be seen that all the images of APO have high qualities.
This corroborates Fig. \ref{fig:demo2}a, where the most images with high GS losses are uploaded.
In contrast, images B and F are blurred for the ranking scheme.
This is because the ranking scheme allocates more resources to frames with higher GS losses. However, several frames experience bad channel conditions, costing excessive power resources as shown in Fig. \ref{fig:demo2}b.
In addition, image E is distorted for the MaxRate scheme. 
This is because the MaxRate scheme allocates excessive resources to frames with good channel conditions for rate maximization as shown in Fig. \ref{fig:demo2}c, ignoring the MR requirements.

\begin{table}[!t]
    \centering
    \caption{Comparison Between MaxImg and APO}
    \label{ssim}
        \scalebox{0.65}{
    \begin{tabular}{ccccccc}
        \toprule
        Metric & \multicolumn{2}{c}{Loss}  & \multicolumn{2}{c}{PSNR} &  \multicolumn{2}{c}{SSIM} \\
        \cline{2-3} \cline{4-5} \cline{6-7} 
        $P$ & MaxImg & APO (Ours) & MaxImg & APO (Ours) & MaxImg & APO (Ours) \\
        \midrule
        5\,mW & $0.04324$ & $0.03766$ & $26.4852$ & $26.9486$ & $0.95488$ & $0.96125$ \\
        6\,mW & $0.04142$ & $0.03694$ & $27.1953$ & $27.6038$ & $0.95670$ & $0.96201$ \\
        7\,mW & $0.04016$ & $0.03537$ & $27.9208$ & $28.3188$ & $0.95812$ & $0.96362$ \\
        8\,mW & $0.03870$ & $0.03411$ & $28.5713$ & $29.0831$ & $0.95992$ & $0.96506$ \\
        \bottomrule
    \end{tabular}
    }
\end{table}

One may wonder whether the superior performance of APO is due to more images uploaded to the server. 
To see this, we further compare the proposed APO method with the following:
\begin{itemize}
    \item \textbf{MaxImg}, a GSMR scheme that maximizes the number of uploaded images.
\end{itemize}
In particular, we consider the case of $P=8$\,mW and $\mathcal{K}=30$\,dB.
It can be seen from Fig.~\ref{fig:maximg}a and Fig.~\ref{fig:maximg}b that the MaxImg scheme uploads more images than our method. 
However, the MR loss, PSNR, and SSIM performances of the MaxImg scheme are worse than those of our APO, as shown in Table II.
This is because the MaxImg waste resources on frames with high GS rendering qualities. 
These frames should have been rendered at the server side without any image uploading. 
Our method is aware of the GS model and \textbf{focus on the frames that cannot be generated by the GS model}. 
We also provide the quantitative results for the cases of $P=\{5,6,7\}$\,mW in Table II.
The APO method consistently outperforms MaxImg by a large margin.

\subsection{Experiment 4: Evaluation of Robust GSCLO}

\begin{figure}[!t]
	\centering
	\begin{subfigure}{0.49\linewidth}
		\centering
		\includegraphics[width=1\linewidth]{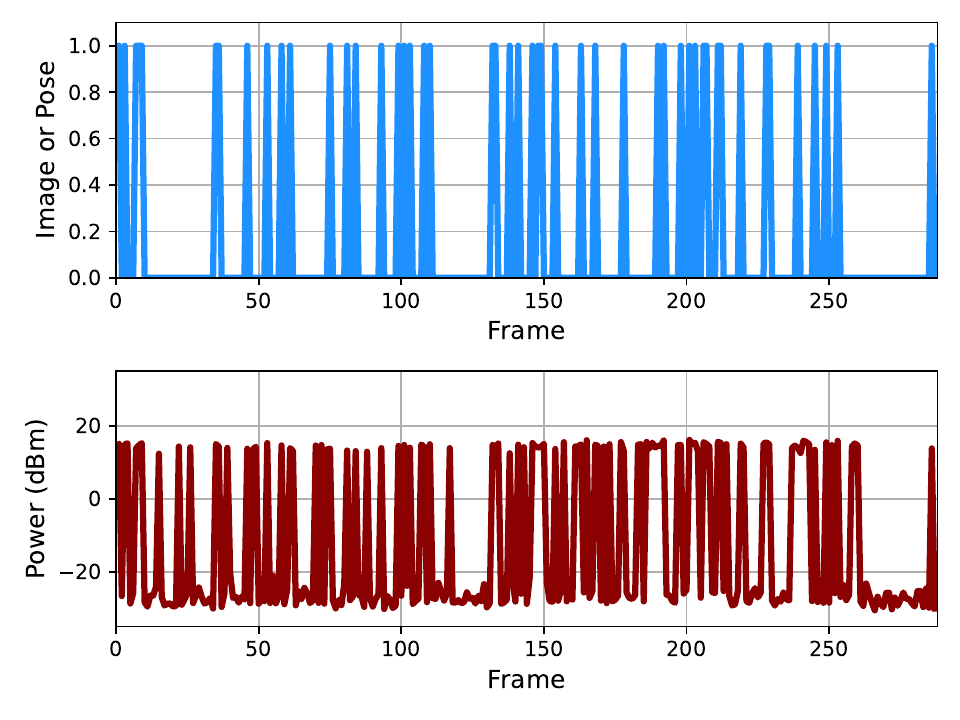}
		\caption{APO-based GSCLO.}
	\end{subfigure}
 	\begin{subfigure}{0.49\linewidth}
		\centering
		\includegraphics[width=1\linewidth]{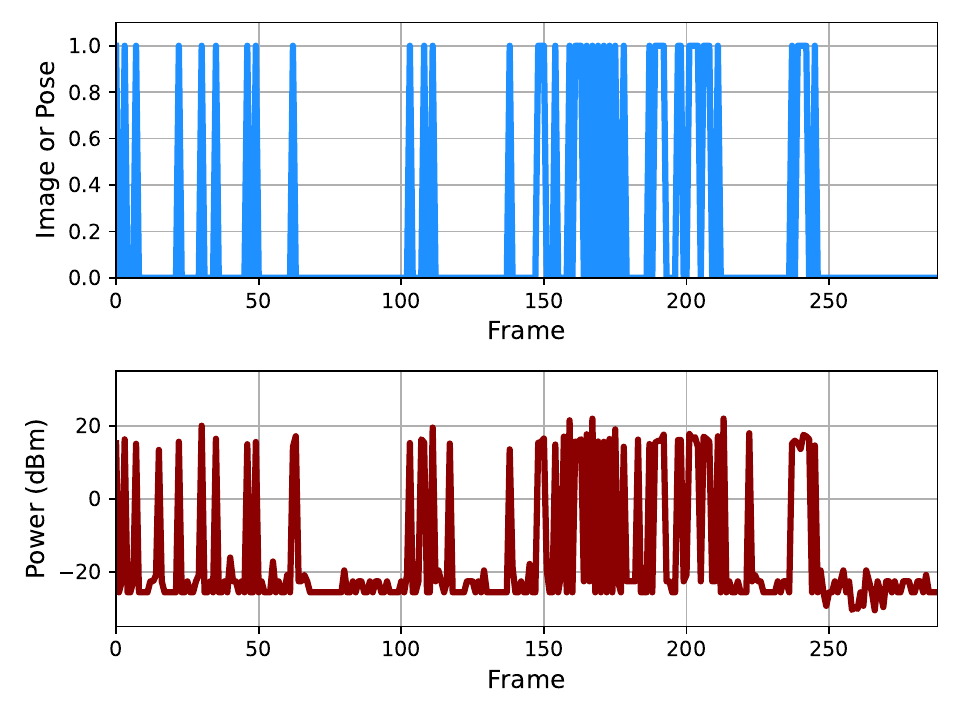}
		\caption{BILS-based robust GSCLO.}
	\end{subfigure}
	\caption{Content switching $\{x_t\}$ and power profiles $\{p_t\}$ under channel uncertainties.}
	\label{robust_demo}
        
\end{figure}

\begin{figure}[t]
    \centering
    \includegraphics[width=0.35\textwidth]{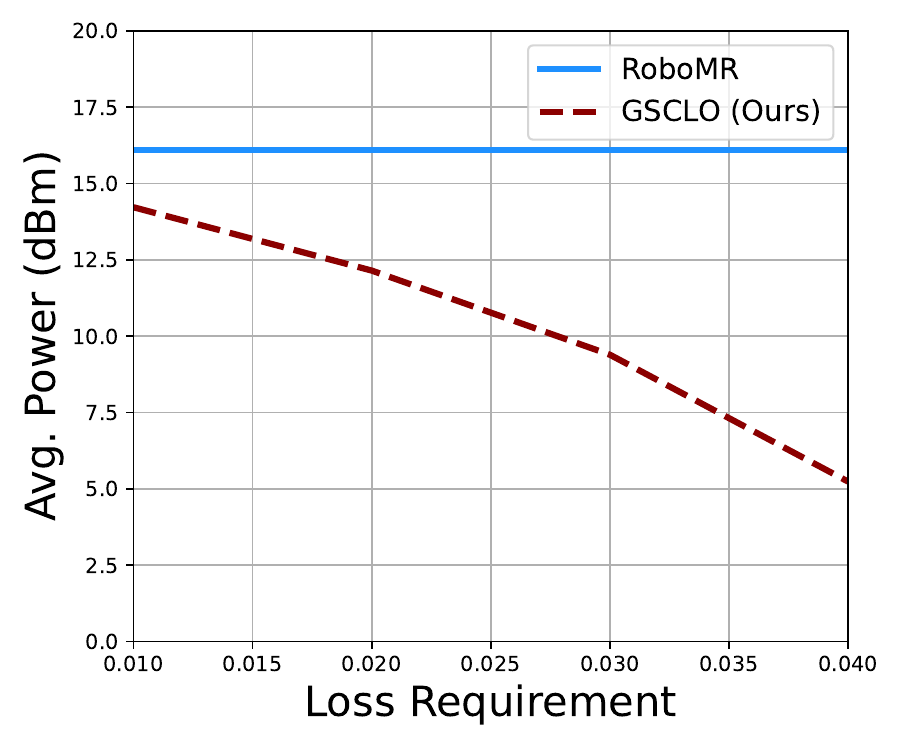}
    \caption{Power consumption.}
    \label{fig:power}
    
\end{figure}

\begin{figure*}[t]
    \centering
    \includegraphics[width=0.98\textwidth]{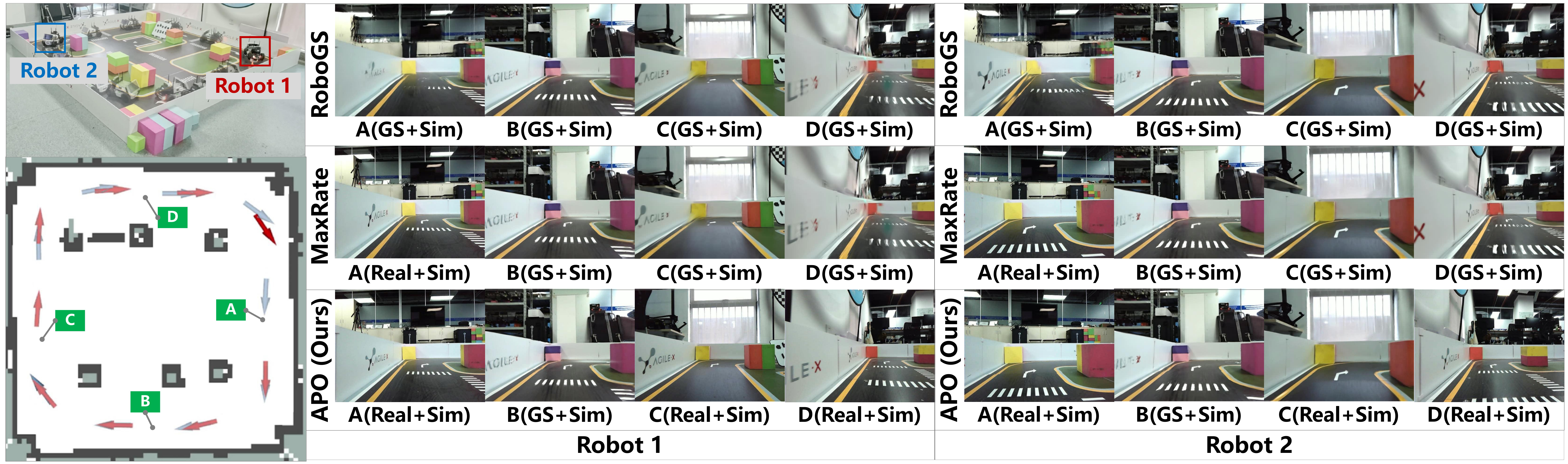}
    \caption{The multi-robot GSMR experiment. Left-hand side: Image positions in the multi-robot MR scenario, where red arrow denotes robot 1 and blue arrow denotes robot 2. Right-hand side: Visualization of robots' images.}
    \label{fig:multi-demo}
\end{figure*}

\begin{table*}[!t]
    \centering
    \caption{Comparison of GS, MaxRate, and APO in Multi-Robot Scenarios}
    \label{ssim}
    \scalebox{0.95}{
    \begin{tabular}{ccccccccccc}
        \hline
        Power & Robot &  \multicolumn{3}{c}{Loss}  & \multicolumn{3}{c}{PSNR} &  \multicolumn{3}{c}{SSIM} \\
        \cline{3-5} \cline{6-8} \cline{9-11} 
        $P$ & ID & GS & MaxRate & APO (Ours) & GS & MaxRate & APO (Ours) & GS & MaxRate & APO (Ours) \\
        \hline
        \multirow{2}{*}{5\,mW}
         & 1 & $0.0987$ & $0.0967$ & $\mathbf{0.0492}$ & $16.676$ & $17.309$ & $\mathbf{33.529}$ & $0.874$ & $0.877$ & $\mathbf{0.937}$ \\
         & 2 & $0.173$ & $0.170$ & $\mathbf{0.0848}$ & $13.258$ & $13.850$ & $\mathbf{31.137}$ & $0.785$ & $0.788$ & $\mathbf{0.896}$ \\
        \multirow{2}{*}{10\,mW} & 1 & $0.0987$ & $0.0779$ & $\mathbf{0.0298}$ & $16.676$ & $23.379$ & $\mathbf{40.522}$ & $0.874$ & $0.901$ & $\mathbf{0.961}$ \\
         & 2 & $0.173$ & $0.140$ & $\mathbf{0.0445}$ & $13.258$ & $20.058$ & $\mathbf{39.796}$ & $0.785$ & $0.826$ & $\mathbf{0.946}$ \\
        \hline
    \end{tabular}
    }
\end{table*}

\begin{figure*}[!t]
	\centering
	\begin{subfigure}{0.25\linewidth}
		\centering
		\includegraphics[width=1\linewidth]{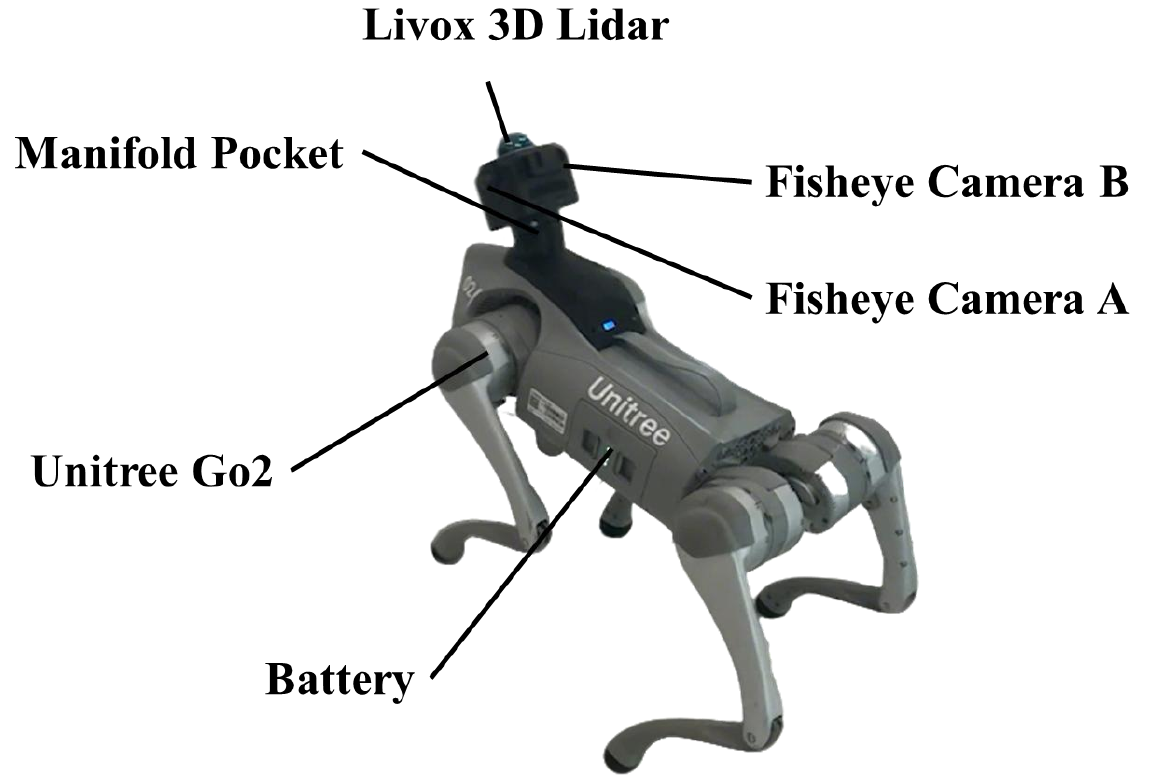}
        \caption{Unitree Go2 robotic platform.}
	\end{subfigure}
	\centering
	\begin{subfigure}{0.36\linewidth}
		\centering
		\includegraphics[width=1\linewidth]{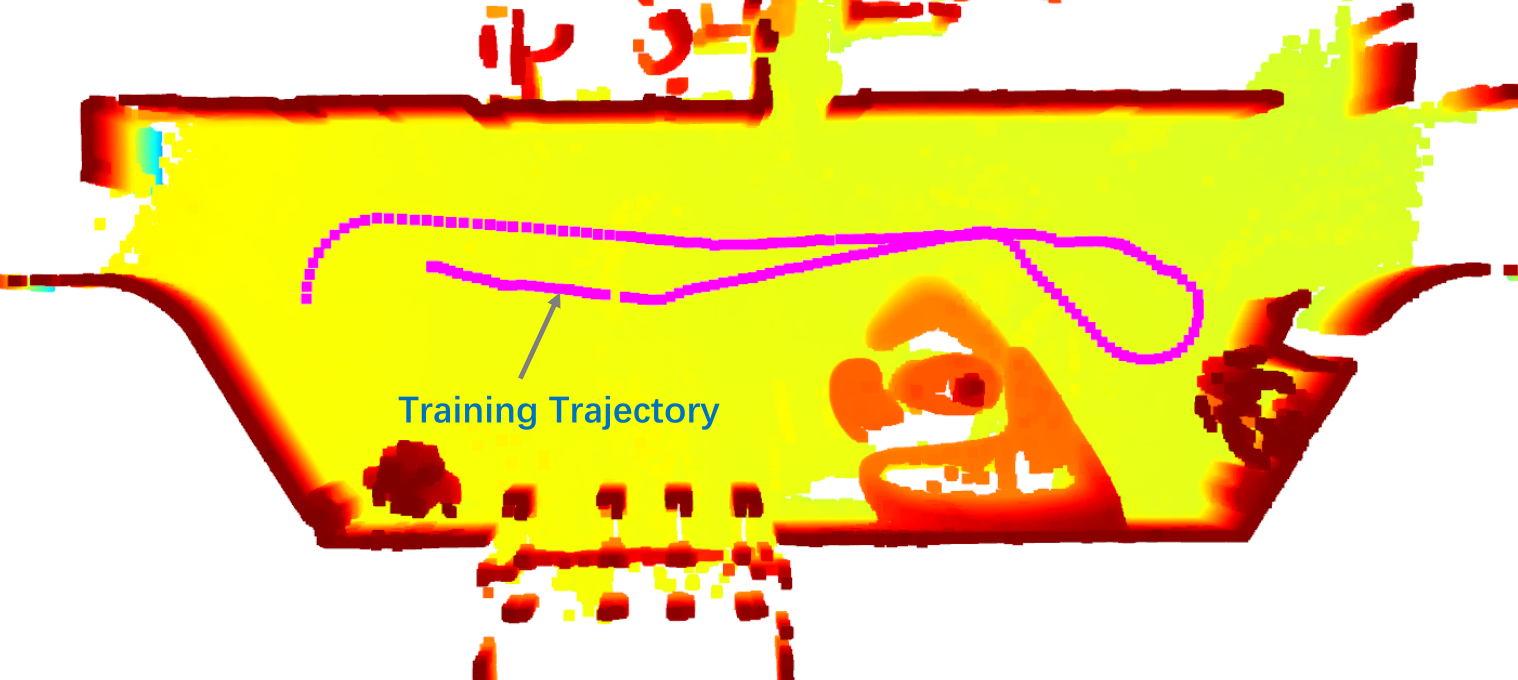}
        \caption{Robot training trajectory.}
	\end{subfigure}
	\centering
	\begin{subfigure}{0.34\linewidth}
		\centering
		\includegraphics[width=1\linewidth]{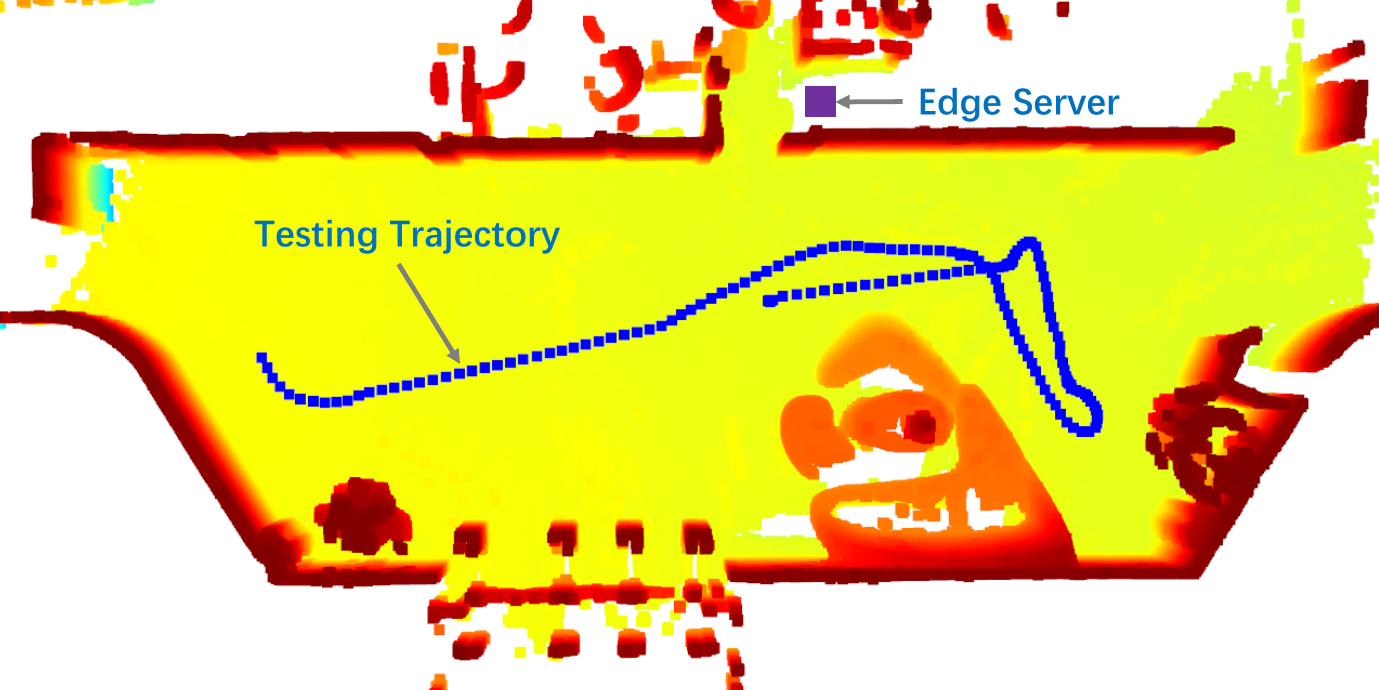}
        \caption{Robot testing trajectory.}
	\end{subfigure}
	\caption{The robotic dog platform and its associated trajectories for training and testing.}
     \label{fig:dog}
     
\end{figure*}

\begin{figure}[!t]
	\centering
	\begin{subfigure}{0.49\linewidth}
		\centering
		\includegraphics[width=1\linewidth]{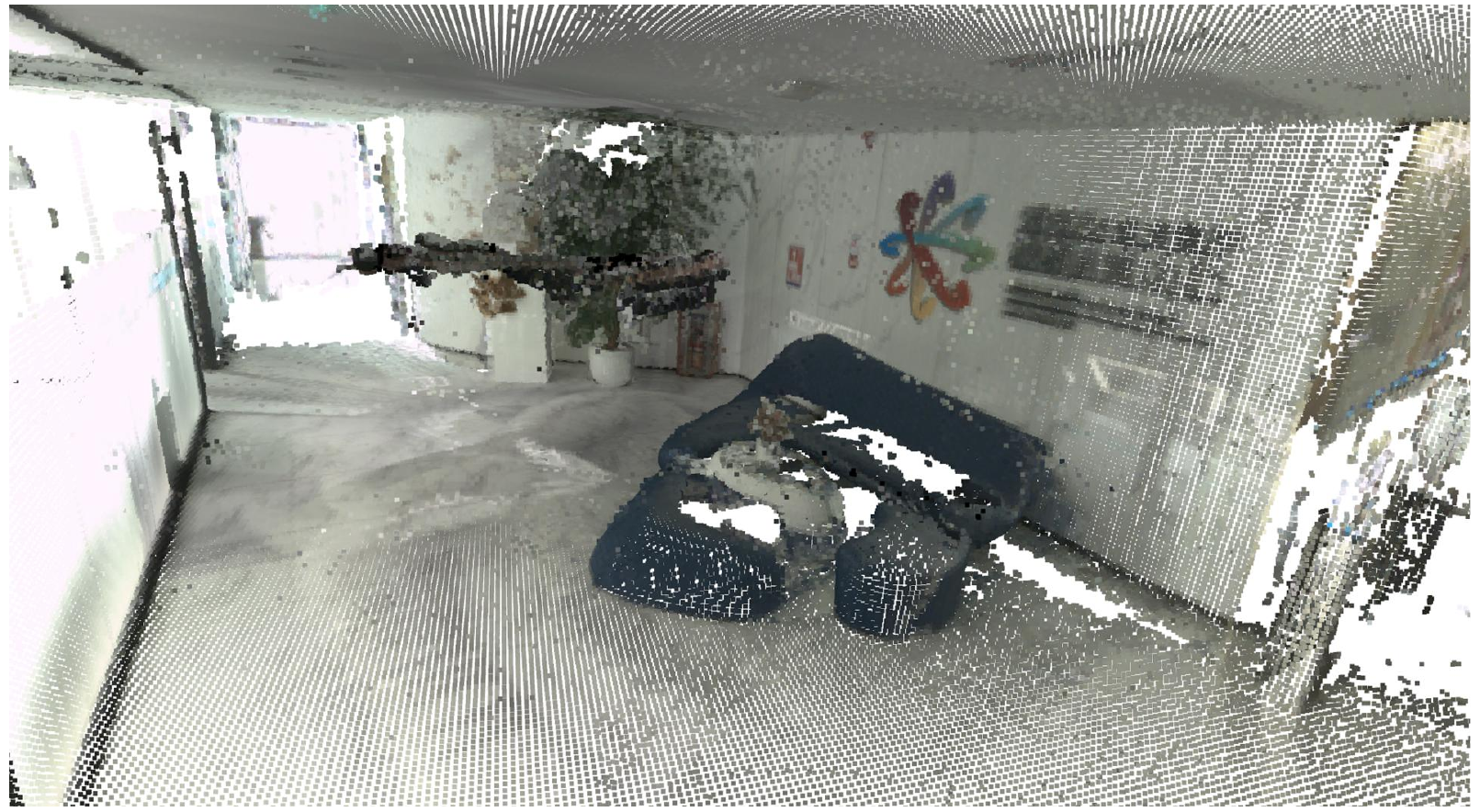}
	\end{subfigure}
	\centering
	\begin{subfigure}{0.49\linewidth}
		\centering
		\includegraphics[width=1\linewidth]{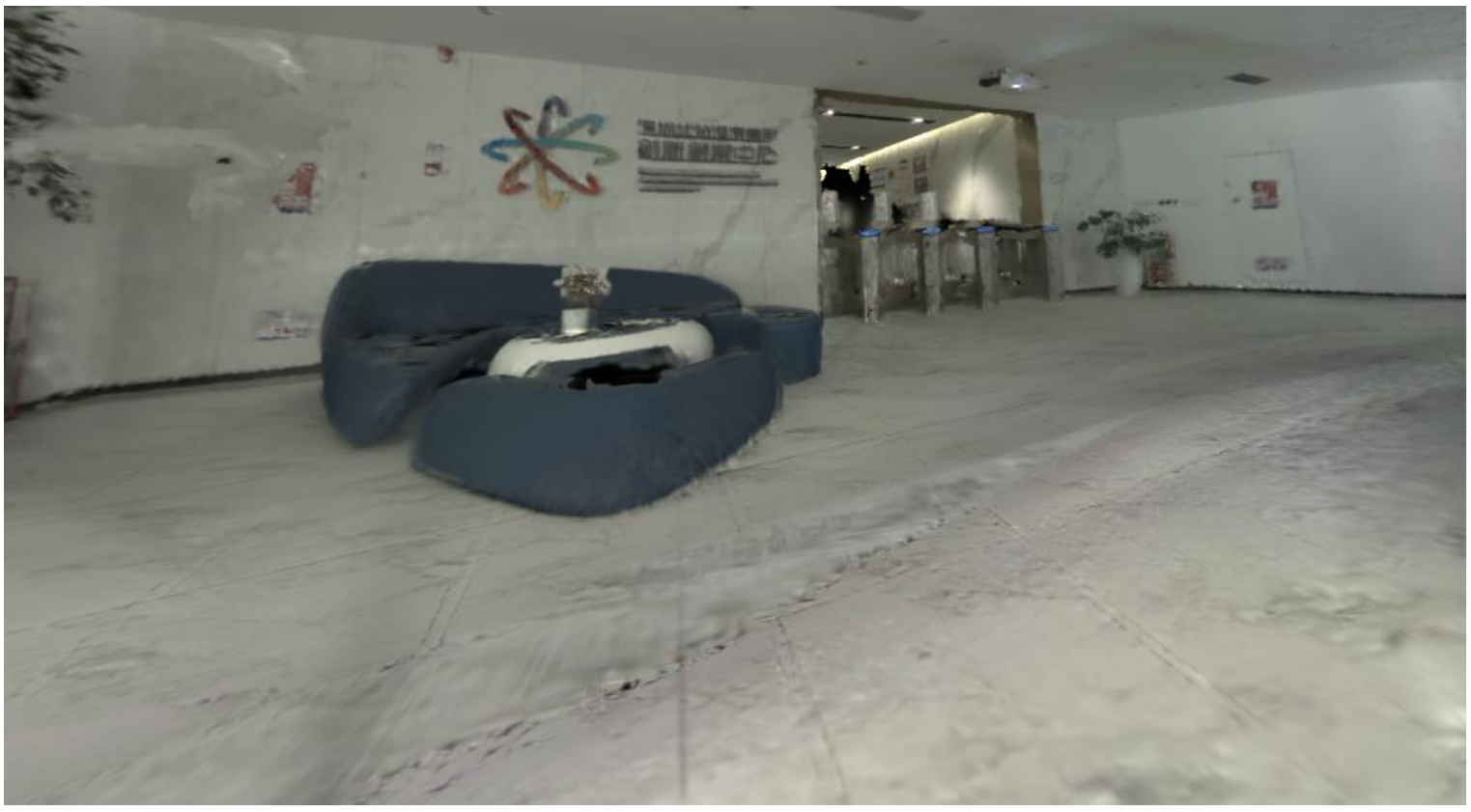}
	\end{subfigure}
	\caption{The GS point cloud and 3D GS model.}
     \label{fig:dog_gs}
     
\end{figure}

\begin{figure*}[!t]
	\centering
	\begin{subfigure}{0.32\linewidth}
		\centering
		\includegraphics[width=1\linewidth]{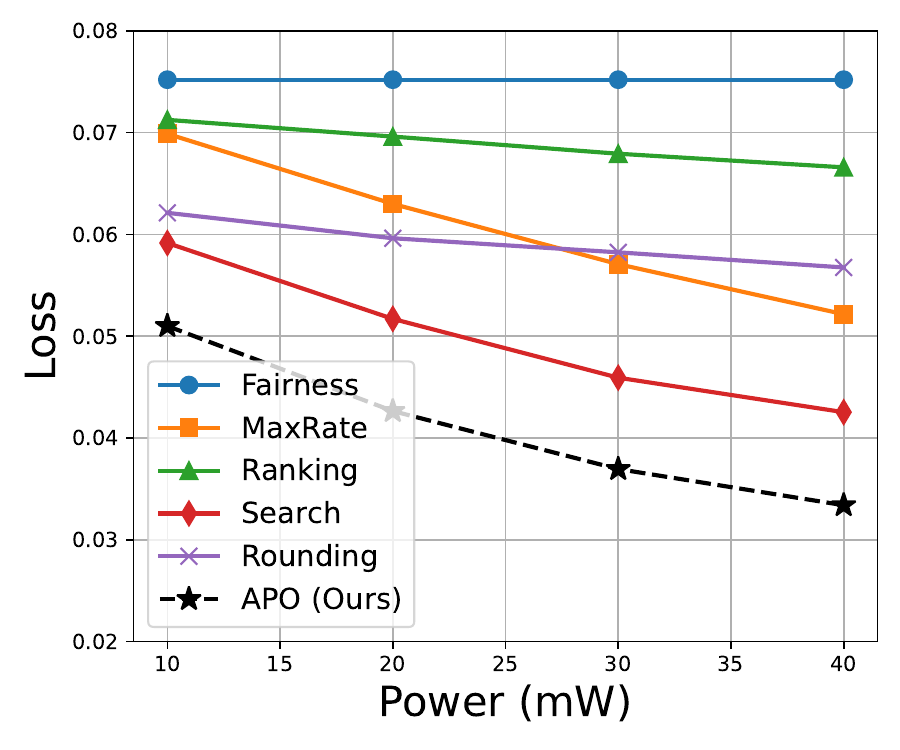}
		\caption{Loss versus $P$.}
	\end{subfigure}
 	\begin{subfigure}{0.32\linewidth}
		\centering
		\includegraphics[width=1\linewidth]{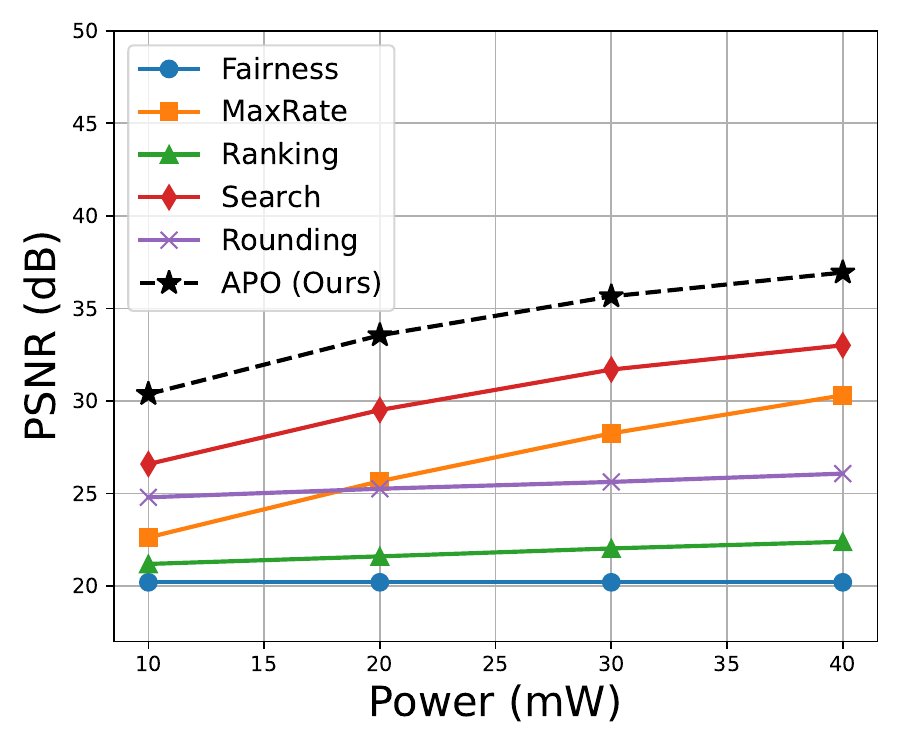}
		\caption{PSNR versus $P$.}
	\end{subfigure}
     \begin{subfigure}{0.32\linewidth}
		\centering
		\includegraphics[width=1\linewidth]{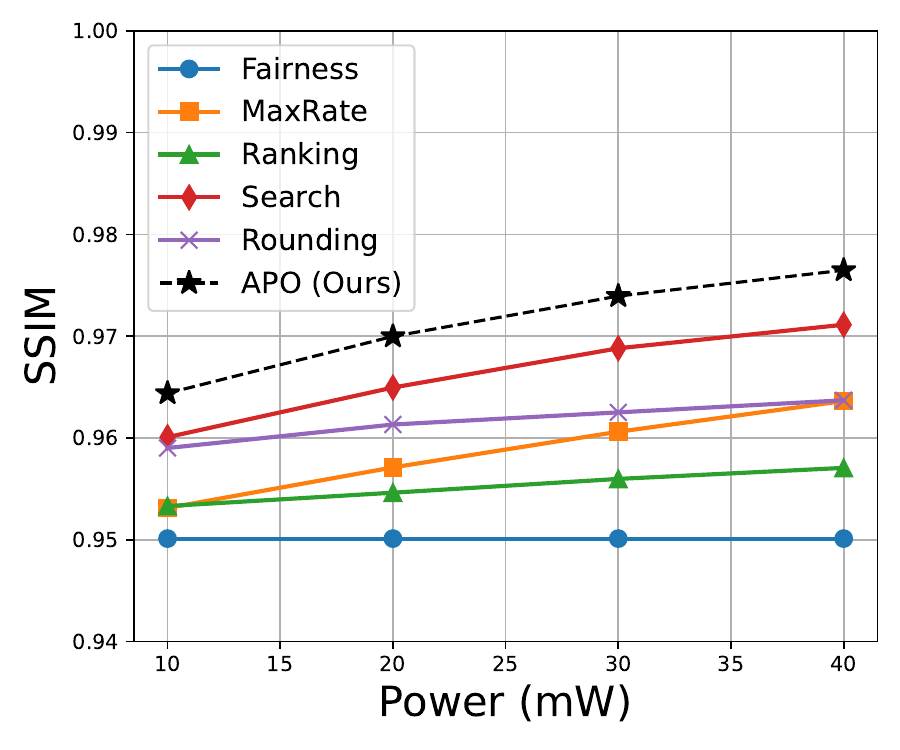}
		\caption{SSIM versus $P$.}
	\end{subfigure}
	\caption{Quantitative comparison between the proposed APO and other benchmarks for the robotic dog experiment 7.}
	\label{dog_compare}
\end{figure*}

\begin{figure*}[t]
    \centering
    \includegraphics[width=0.98\textwidth]{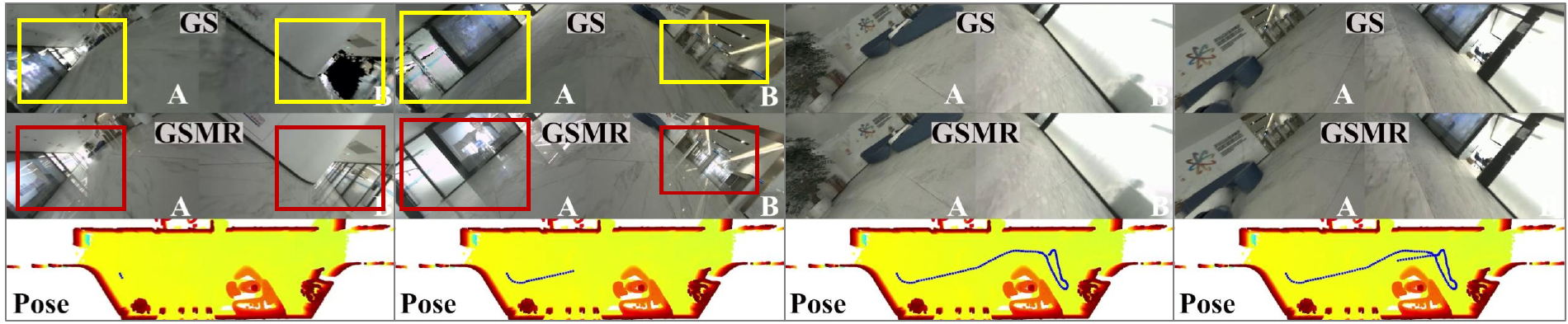}
    \caption{Visualization for the robotic dog experiment.}
    \label{dog_demo}
\end{figure*}

To evaluate the robustness of GSCLO against channel uncertainties, we compare the proposed BILS in Algorithm 2 of Section V and the APO in Algorithm 1 of Section IV. 
We consider the case of $P=\{10,20,30,40\}$\,mW and $\mathcal{K}=10$\,dB. 
The mean of channel $\Tilde{h}_t$ is generated using \eqref{ht}, 
and the channel uncertainty is set to $\omega^2=0.04 |\Tilde{h}_t|^2 $. 
The threshold of chance constraint \eqref{Pc_GS'} is set to $\epsilon=0.1$.
First, by comparing Fig.~\ref{compare2} and Fig.~\ref{robust}, it can be seen that the performance of the proposed APO is degraded when there exists errors in channel estimation.
Second, the proposed robust GSCLO with BILS (i.e., Algorithm 2) consistently outperforms the proposed GSCLO with APO (i.e., Algorithm 1), in terms of both loss and PSNR performances. 
In particular, the loss is reduced by about $10\%$, and the PSNR is improved by about $2$\,dB.
This is because the BILS-based robust GSCLO \textbf{boosts the transmit powers of key frames to ensure their successful transmission}, by reducing the number of image uploads as shown in Fig.~\ref{robust_demo}. 
In contrast, the APO-based GSCLO allocates exact transmit powers satisfying constraint \eqref{Pc_GS}, which may fail in satisfying \eqref{Pc_GS} when $h_{t} = \Tilde{h}_{t} + \Delta h_t$ as in \eqref{uncertainty}, thus failures of transmission when the channel is deteriorated. 
This also corroborates Fig.~\ref{robust}c, where the packet loss probabilities of the proposed robust GSCLO with BILS are strictly below $\epsilon=0.1$, while the APO GSCLO leads to higher packet loss probabilities.
The above results demonstrate that the proposed BILS-based robust GSCLO is more resilient to channel uncertainties than the APO-based GSCLO.

\subsection{Experiment 5: Power Minimization GSCLO}

To evaluate the power saving brought by GSCLO, we compare the proposed power minimization GSCLO in Section VI-A and RoboMR. 
The average transmit power versus the MR loss requirement $L_{\mathrm{th}}$ at $\mathcal{K}=30$\,dB is shown in Fig.~\ref{fig:power}.
According to \cite{kerbl20233d}, to ensure satisfactory images, the PSNRs should be controlled between $30$\,dB and $40$\,dB. 
Further leveraging the results in Experiment 2, a threshold $L_{\mathrm{th}}$ between $0.01$ and $0.04$ would be a proper choice for constraining the MR loss in practice. 
Under $L_{\mathrm{th}}\in[0.01,0.04]$, it can be seen from Fig.~\ref{fig:power} that the power saved by GS is at least $2$\,dB.
Particularly, when the loss threshold is $0.03$, the power saved by GSCLO is larger than $6$\,dB, i.e., equal to $75\%$ power reduction.
This result reveals that \textbf{robotic mixed reality can be achieved with ultra-low communication costs ($<200$\,bits/frame) if Gaussian splatting is leveraged properly}. 
With ultra-low communication costs, it is possible to connect massive robots with the virtual world, forming a metaverse space.

\subsection{Experiment 6: Evaluation of Multi-Robot GSCLO}

To evaluate the performance of the proposed method in Section VI-B under multi-robot settings, we consider the case of $N=4$ and $K=2$ with $T=433$ in the AGX scenario shown in Fig.~\ref{fig:multi-demo}.
We compare APO with the MaxRate and GS schemes under the settings of $P\in\{5,10\}$\,mW and $\mathcal{K}=0$\,dB.
The loss, PSNR, and SSIM results are shown in Table III.
Qualitative results are shown in Fig.~\ref{fig:multi-demo}.
It can be seen that the solution maintains efficiency with more robots, by mitigating the co-channel interference among different devices via multi-antenna techniques.
This also demonstrates that our method works well under longer operation time.

\subsection{Experiment 7: Robotic Dog in Larger Environments}

Finally, we conduct experiments using the Unitree Go2 robotic dog platform shown in Fig.~\ref{fig:dog}a in a larger environment. 
We navigated the dog for two rounds in an office building. 
In the first round, the robot collects $H=400$ samples (including images and poses) for training a GS model. 
The trajectory of the training data is shown in Fig.~\ref{fig:dog}b.
The obtained GS point cloud and GS model are shown in Fig.~\ref{fig:dog_gs}.
In the second round, the robot collects $T=370$ samples for evaluation.
The trajectory of the testing data is shown in Fig.~\ref{fig:dog}c.
It can be seen that the testing poses are random and different from the training poses. 
Each image has a data volume of $I=13.84$\,Mbits ($1.73$\,MBytes).

We simulate the case of $B=20$\,MHz in non-LoS settings, where the server is marked as purple box in Fig.~\ref{fig:dog}c. 
The channel is assumed to be Rayleigh fading, i.e.,
\begin{align}
&h_t=\sqrt{\varrho_0 \omega_0 d_t^{-\alpha}}g_t^{\mathrm{NLOS}},
\end{align}
where $\omega_0=-10\,$dB is the wall blockage fading, $d_t$ is computed based on the robot trajectory and the server position. 
The loss, PSNR, and SSIM results are shown in Fig.~\ref{dog_compare}. 
It can be seen that no matter how power budget changes, the proposed APO outperforms other benchmarks by a large margin in terms of all the performance metrics.

Qualitative results are shown in Fig.~\ref{dog_demo}, where captions A and B represent fisheye cameras A and B.
Inside the view of camera A, the screen is playing video and the environment is changing. 
It can be seen that all the images of our APO-based GSMR have high qualities.
This is because if the robot detects a significant discrepancy between the real captured image and the GS rendered view (i.e., “memory bias”)\footnote{
In practice, such discrepancies can be estimated using pilot images.}, our method would trigger an image upload. 
In contrast, the existing GS scheme works poorly in dynamic scenarios. 
This is due to the inherent limitation of static reconstruction. 
This result suggests that \textbf{mixture of data is useful for enhancing GS performance in dynamic scenarios}.

\section{Conclusion}\label{section8}

This paper presented GSMR, which realizes low-cost communication between the simulation server and the robot via the introduction of GS. 
A GSCLO framework was proposed to mitigate the discrepancies between GS-rendered images and real-world environments.
Extensive experiments were conducted, which demonstrate the effectiveness, stability, and robustness of GSMR on wheeled-robot, legged-robot, and multi-robot platforms in various scenarios. 
It is found that the GS model indeed helps MR, but we need to carefully avoid memory bias. 
Furthermore, under channel uncertainties, we should boost transmit powers of key frames, so as to ensure successful pose and image uploads. 

\bibliographystyle{IEEEtran}

\end{document}